\documentclass[twoside,11pt]{article}

\usepackage{blindtext}

\usepackage{jmlr2e}

\usepackage{lastpage}
\jmlrheading{24}{2023}{1-\pageref{LastPage}}{10/22; Revised
11/23}{12/23}{22-1190}{Zihao Li, Boyi Liu, Zhuoran Yang, Zhaoran Wang, Mengdi Wang}
\ShortHeadings{Primal-Dual Policy Optimization for Constrained Reinforcement Learning}{Li, Liu, Yang, Wang, Wang}

\firstpageno{1}
\def\*#1{\bm{#1}} 
\def\+#1{\mathcal{#1}} 
\def\-#1{\mathrm{#1}} 
\def\^#1{\mathbb{#1}} 
\def\$#1{\mathtt{#1}}
\def\@#1{\mathsf{#1}}
\usepackage{enumerate}
\usepackage{algorithm}
\usepackage{algorithmic}
\usepackage[OT1]{fontenc}
\usepackage{hyperref}
\usepackage{dsfont}
\usepackage{bm}

\usepackage{amsmath,amssymb}
\usepackage{natbib}
\usepackage{scalerel}
\usepackage[usenames]{color}
\usepackage{CJK}
\usepackage{multirow}
\usepackage{breakcites}
\def\##1\#{\begin{align}#1\end{align}}
\def\$#1\${\begin{align*}#1\end{align*}}
\def\given{\,|\,}

\def\tr{\mathop{\text{tr}}\kern.2ex}

\def\P{{\mathbb P}}

\let\hat\widehat
\let\tilde\widetilde

\def\given{{\,|\,}}

\long\def\comment#1{}

\def\tr{\mathop{\text{Tr}}}

\def\cS{{\mathcal{S}}}

\def\cX{{\mathcal{X}}}
\def\cD{{\mathcal{D}}}
\def\cP{{\mathcal{P}}}
\def\cL{{\mathcal{L}}}
\def\cN{{\mathcal{N}}}

\def\cB{{\mathcal{B}}}
\def\tr{{\text{Tr}}}

\def\*#1{\bm{#1}} 
\def\+#1{\mathcal{#1}} 
\def\-#1{\mathrm{#1}} 
\def\^#1{\mathbb{#1}} 
\def\$#1{\mathtt{#1}}
\def\@#1{\mathsf{#1}}

\def\dual{\@D(\bm{\alpha},\bm{\beta})}

\newcommand{\bel}{\begin{eqnarray}\label}
\newcommand{\eel}{\end{eqnarray}}
\newcommand{\bes}{\begin{eqnarray*}}
\newcommand{\ees}{\end{eqnarray*}}

\def\real{{\mathbb{R}}}
\def\R{{\real}}

\newcommand{\red}{\color{black}}

\newcommand{\cA}{\mathcal{A}}
\newcommand{\cC}{\mathcal{C}}
\newcommand{\cE}{\mathcal{E}}
\newcommand{\cF}{\mathcal{F}}
\newcommand{\cG}{\mathcal{G}}
\newcommand{\cH}{\mathcal{H}}
\newcommand{\cI}{\mathcal{I}}

\newcommand{\cM}{\mathcal{M}}
\newcommand{\cO}{O}

\newcommand{\cV}{\mathcal{V}}
\newcommand{\cW}{\mathcal{W}}

\newcommand{\EE}{\mathbb{E}}

\newcommand{\PP}{\mathbb{P}}

\newcommand{\RR}{\mathbb{R}}

\newtheorem{assumption}[theorem]{Assumption} 

\begin{document}

\title{Double Duality: Variational Primal-Dual Policy
Optimization for Constrained Reinforcement Learning}

\author{\name Zihao Li \email zihaoli@princeton.edu \\
       \addr Department of Electrical and Computer Engineering\\
       Princeton University\\
       Princeton, NJ 08544, USA
       \AND
       \name Boyi Liu \email boyiliu2018@u.northwestern.edu \\
       \addr Department of Industrial Engineering and Management Sciences\\
       Northwestern University \\
        IL 60208, USA
       \AND
       \name Zhuoran Yang \email zhuoranyang.work@gmail.com \\
       \addr Department of Statistics and Data Science\\
       Yale University \\
        CT 06511-6814, USA
        \AND
       \name Zhaoran Wang \email zhaoranwang@gmail.com \\
       \addr Department of Industrial Engineering and Management Sciences\\
       Northwestern University \\
        IL 60208, USA
      \AND
      \name Mengdi Wang \email mengdiw@princeton.edu \\
       \addr Department of Electrical and Computer Engineering\\
       Princeton University\\
       Princeton, NJ 08544, USA
               }

\editor{Mingyuan Zhou}
\maketitle

\begin{abstract}%   <- trailing '%' for backward compatibility of .sty file
		 We study the Constrained Convex Markov Decision Process (MDP), where the goal is to minimize a convex functional of the visitation measure, subject to a convex constraint.  Designing algorithms for a constrained convex MDP faces several challenges, including (1) handling the large state space, (2) managing the exploration/exploitation tradeoff, and (3) solving the constrained optimization where the objective and the constraint are both nonlinear functions of the visitation measure. In this work, we present a {\red{model-based algorithm,  }}\underline{V}ariational \underline{P}rimal-\underline{D}ual \underline{P}olicy \underline{O}ptimization (VPDPO), in which Lagrangian and Fenchel duality are implemented to reformulate the original constrained problem into {\red{an}} unconstrained primal-dual optimization. Moreover, the primal variables are updated by {\red{model-based}} value iteration following the principle of \textit{Optimism in the Face of Uncertainty} (OFU),  while the dual variables are updated by gradient ascent. Moreover, by embedding the visitation measure into a finite-dimensional space, we can handle large state spaces by incorporating function approximation. Two notable examples are (1) Kernelized Nonlinear Regulators and (2) Low-rank MDPs. We prove that {\red{with an optimistic planning oracle,}} our algorithm achieves sublinear regret and constraint violation in both cases and can attain the globally optimal policy of the original constrained problem.
\end{abstract}


\section{Introduction}

In recent years, constrained reinforcement learning (RL) has attracted greater research interest. In contrast to unconstrained RL, in which an agent can freely learn to maximize its cumulative reward or minimize its cost by interacting with an unknown environment, we face learning problems with various kinds of constraints in many real-world applications. For example, in autonomous driving, we want to minimize the time cost while avoiding speeding or colliding with other cars \citep{garcia2015comprehensive}. Other applications include cost-constrained RL in medical applications and business restrictions for tax collection optimization \citep{10.1145/1835804.1835817}, in which the total budget is restricted. 

However, existing works on Markov decision process (MDP) with constraints are still limited. Currently, most works consider the constrained MDP with both the objectives and constraints being linear functionals of visitation measures \citep{efroni2020exploration, ding2021provably}. However, in many complex scenarios, we encounter problems with certain nonlinear structures. For example, in apprenticeship learning the agent aims to simulate the performance of an expert in a demonstrated task \citep{apprenticeship}. It is difficult to formulate an explicit reward function, and the learning goal is given by the $\ell_2$-norm distance between the visitation measure of the agent and the expert. In multi-objective MDP, we have to consider nonlinear interaction between different objectives \citep{wu2021offline, yu2021provably}. Other examples include cautious MDP \citep{zhang2020cautious} and general utility MDP \citep{zhang2020variational}.

In this work, we introduce the Constrained Convex Markov Decision Process (C$^2$MDP), where we consider a constrained convex optimization over the space of visitation measures. The agent manipulates her policy over the space of visitation to minimize the objective while fulfilling the constraints. Compared to previous works, our model allows objectives and constraints to be nonlinear in visitation measure, thus significantly extending beyond Constrained MDP \citep{efroni2020exploration, ding2021provably}. Moreover, our model covers interesting examples such as convex MDP \citep{zahavy2021reward}, general utility RL \citep{zhang2020variational}, and apprenticeship learning \citep{apprenticeship} as special cases. 
Challenges in designing an efficient online algorithm for constrained convex MDP are threefold:
\begin{enumerate}[(i)]
    \item Most existing theoretical convergence guarantees for convex MDP apply only to the tabular case \citep{zhang2020variational, efroni2020exploration,zahavy2021reward}, where the visitation measure is a vector of dimension $O(H|\cS||\cA|)$, making the convex MDP a convex optimization problem. However, when facing a continuous state space, the visitation measure becomes a general distribution on the state-action space. Due to the curse of dimensionality, algorithms designed for tabular MDP fail to tackle the problem.
    \item Highly different from simple constrained MDP, which only imposes a linear constraint in the value function, the objective and constraint of C$^2$MDP  can be nonlinear functionals of the visitation measure. Without knowing further structure, finding optimal solutions for such problems is much harder than Constrained MDP, which is equivalent to solving a linear programming problem \citep{efroni2020exploration}.
    \item In a C$^2$MDP, the transition of the environment is unknown, and can only be learned through the transition through interacting with the environment. With limited information, designing an efficient online exploration strategy is hard.
\end{enumerate}
With these coupled challenges, we ask the following question:
\begin{center}
	\textit{Can we find the globally optimal policy of constrained convex MDP in online learning?}
\end{center}
In this work, we give an affirmative answer to this question.
\begin{itemize}
    \item To handle (i), we incorporate function approximation and formulate the optimization in the embedded space of the visitation measures. In particular, we consider the feature map in function approximation and its expectation under the visitation measure, which is known as the kernel embedding of visitation measure \citep{2008kernel,muandet2016kernel}. We further consider the optimization with the kernel embedding of the visitation being the decision variables, which motivates us to implement online optimization techniques for solving C$^2$MDP. Such a formulation recovers the tabular setting as a special case when using the canonical embedding. 
    \item To handle (ii), we use Lagrangian duality to transform the constrained problem to an unconstrained minimax optimization problem. In presence of Slater's condition, it is guaranteed that the original minimization shares the same optimal value with the unconstrained one. Moreover, to handle nonlinearity in the objective and the constraint, we apply Fenchel duality to introduce a linear structure. Combining the above two types of duality, we obtain a primal-dual optimization problem with a linear dependency on the kernel embedding. This allows us to construct a linear reward and adopt techniques of {\red{previous works in model-based value iteration, such as \cite{kakade2020information, ayoub2020model}.}}
    \item To handle (iii), we apply the principle of \textit{Optimism in the Face of Uncertainty} (OFU) \citep{jin2020provably,yang2020provably} {\red{by an optimistic planning oracle \citep{jin2021bellman, kakade2020information,ayoub2020model}}} which behaves as if the model parameters assume their best possible values in accordance to the observations so far. 
\end{itemize}

With the above techniques, our algorithm is provably sample-efficient. In specific, we prove that our algorithm achieves $\cO(\sqrt{T})$ in both the regret and the constraint violation, where $T$ is the number of the sampling episodes.
To the best of our knowledge, our algorithm is the first provably sample-efficient algorithm for the constrained nonlinear optimization over visitation measures. As special cases, our method can be widely applied to multi-objective MDP, and apprenticeship learning, and lead to efficient algorithms.
\subsection{Related Works}
\textbf{Optimization over occupancy measures/Convex MDP.} Several early works \citep{tewari2007optimistic,chen2016stochastic, wang2017primal, wang2020randomized} studied tabular MDP via linear programming reformulation.  \cite{zahavy2021reward} studied convex MDP via Fenchel duality. \citet{zhang2020cautious,zhang2020variational,zhang2020variational} studied for convex optimization over occupancy measures. However, while all these methods are successful in tabular MDP, they cannot (i) avoid the curse of dimensions in large state space MDP, and (ii) handle constraints. \\
\textbf{Constrained MDP.} Our work is a generalization of the constrained MDP.  \citet{efroni2020exploration,yu2021provably,qiu2020upper,brantley2021constrained} studied tabular constrained MDP. \cite{ding2021provably}  studied safe reinforcement learning under function approximation setting under a linear mixture MDP model and using upper confidence bound (UCB) algorithm for exploration. \cite{wu2021offline} further provided a general algorithm for Multi-objective MDP with general constraints and objective relies on multiple value functions. All of these methods assume that a given reward exists and explores the environment following the principle of optimism, and achieves great success by providing sublinear regret and constraint violation. \cite{vaswani2022nearoptimal} provides a zero-constrained algorithm and provide a lower bound under such scenario. However, when there is no given reward function, these methods are no longer applicable.\\
\textbf{Provably efficient online RL.} Our work is closely related to a line of provably efficient online RL algorithms on Low-rank MDPs \citep{agarwal2020flambe,uehara2022representation} and kernelized nonlinear regulator \citep{kakade2020information,mania2020active}, where efficient exploration of the agent is obtained by  choosing an optimistic model in the confidence set. However, these results are only designed for unconstrained problems that are linearly dependent on the occupancy measure.
\subsection{Notations}
We denote by $[a: b]$ the set of integers between $a$ and $b$, i.e., $[a: b]=\{i \in \mathbb{Z} \mid a \leq i \leq b\}$, and write $[n]=[1: n] $. We denote by $ {x} = (x_h)_{h\in[H]} $ the column vector obtained by concatenating the elements of $\{x_h\}_{h \in [H]}$, i.e., $ {x} = (x_1;\cdots;x_H) $.  We write $ {a}\cdot{b} $ as the inner product of two finite dimensional vectors, and $ \langle f,g\rangle_{\cH}  $ as the inner product of two functions $ f $ and $ g $ in the reproducing kernel Hilbert space (RKHS) $ \cH $. We also denote by $\|\cdot\|_2$ the $\ell_2$-norm in Euclidean space, and $ \cB^d $ the the unit ball in $ \RR^d $, i.e., $ \{{x}\in \RR^d: \|{x}\|_2\leq1 \} $. The set of probability distribution over a space $ \cX $ is denoted by $\Delta(\cX)$. We define $\cP(s'\,|\, s,a)$ as the probability for the agent transiting to state $s'$ from $s$ when taking action $a$.
\section{Background}\label{sec:background}
In this section, we briefly introduce the concepts of reinforcement learning, Constrained Convex MDPs, Low-rank MDPs, and Kernalized nonlinear Regulator (KNR).
\vskip4pt
\subsection{MDP Setting}
We consider an episodic  Markov decision process problem $(\cS, \cA, H, c)$ , where $\cS \subset \R^{d}$ is the state space embedded in the Euclidean space, $\cA$ is a (possibly continuous) action space, $ H $ is the horizon, and $c=\{c_h\}_{h=1}^{H}$ is a collection of cost functions where $c_h: \cS\times\cA \to \RR$ is the cost of stage $ h $. In each episode, we consider an agent with policy $\pi = \{\pi_h\}_{h=1}^{H}$, where $ \pi_h: \cS \to \Delta(\cA) $. At the stage $h$, the agent takes an action $ a_h \in \cA $ according to the policy $ \pi_h(\cdot\,|\,s_h) $. The state then transits to $ s_{h+1} $ with probability $\cP_h(s_{h+1}\,|\,s_h,a_h)$ according to the underlying transition rule. Since the choice of the initial state does not add complexity to the problem, for simplicity, we assume that the initial state is fixed, i.e., $ s_1 = \overline{s} $.

We introduce the concepts of action-state value function and state value function from reinforcement learning. The action-state value function $ Q_h^{\pi}:\cS \times \cA \to \RR $  is defined as 
\begin{align*}
Q_h^{\pi}(s,a) = \EE_{\pi}\biggl[\sum^H_{i=h}  c_i(s_i, a_i)\,\bigg|\,s_h=s, a_h=a\biggr], \quad\forall (s,a,h) \in \cS\times \cA\times[H].
\end{align*}
Correspondingly, the action-value function $V_h^\pi :\cS\to \RR$ is defined as
	\#\label{eq:v}
	V_h^{\pi}(s) = \EE_{\pi}\biggl[\sum^H_{i = h}  c_i(s_i, a_i)\,\bigg|\,s_h = s \biggr], \quad\forall (s,h) \in \cS\times[H].
	\#
Here the expectation $\EE_{\pi}[\cdot]$ is taken over the trajectory $\{(s_h,a_h)\}_{h\in[H]} $ induced by $\{\pi_h\}_{h\in[H]}$ and the underlying transition. For notation simplicity, we also write
\begin{align*}
\PP_hf(s_h,a_h) = \EE_{s^\prime \sim \cP(\cdot\,|\,s_h,a_h)} [f(s^\prime)]
\end{align*}
for any integrable $f: \cS \rightarrow \RR$ and conditional probability $\cP_h(\cdot\mid s_h,a_h)$.
\subsection{Constrained Convex MDP }
 We generalize the problem of convex  MDP  \citep{zahavy2021reward}, which considers a non-constrained convex optimization problem with the occupancy measure $d_\pi =(d_{\pi,h}(s,a))_{(s,a,h)\in\cS\times\cA\times[H]}$ as the variable. The agent manipulates the occupancy measure by properly adjusting its policy. The aim is to find the optimal policy that minimizes the objective function. A tabular Constrained Convex MDP (C$^2$MDP) usually takes the form of
\begin{align}\label{eq:tab-cvxmdp}
 \min_{\pi\in\Delta(\+A\given\+S,H)} 
	f\big((d_{\pi,h})_{h\in[H]}\big)
	\quad
	\text{s.t.}
	\quad
	&d_{\pi,h}(s,a) = \EE_\pi\big[\mathds{1}_{(s_h,a_h)=(s,a)}\big],\\ 
 &g(d_{\pi,h}(s,a))\leq 0, \nonumber
\end{align}
 where $f, g : \RR^{|\cS||\cA|H} \rightarrow \RR$ are both convex functions. \\
In the tabular MDP, the set of all $x$ induced by the agent's policy is represented by a polytope represented by $\cO(|\cS||\cA|)$ linear constraints \citep{efroni2020exploration, zahavy2021reward}. However, in the continuous state space case, due to the curse of dimensionality and the shortage of memory, such an LP formulation is generally impossible. Therefore, we incorporate function approximation to handle the large state space by embedding the information of the state-action pair with a finite-dimensional feature map $\psi:\cS\times\cA\rightarrow \RR^d$. Such a method is widely used in RL literature \citep{yang2020provably,jin2020provably,uehara2022representation,kakade2020information}. To describe the visitation of the agent, we apply kernel embedding of probability distribution to the visitation measure \citep{muandet2017kernel, zahavy2021reward, efroni2020exploration} . By embedding the probability distribution induced by the agent's policy $\pi$ on $\cS\times\cA$  into finite dimension linear space, the objective and constraints related to the distribution can be reformulated into a function for the kernel embedding. 
%Such an embedding is frequently used in the literature on machine learning  (e.g., \cite{smola2007hilbert},\cite{hsu2019bayesian},\cite{ward2021linear}). 
%  In this paper, we consider a (possibly continuous) MDP with an underlying kernel structure. That is, there exists an RKHS $\cH$ and a kernel feature mapping $\psi:\cS\times\cA \rightarrow \cH$. Our goal is to solve an optimization problem with a Lipschitz-continuous objective function $ f(x) $, where $ x $ is associated with the agent's policy and the kernel feature, meanwhile constrained by a convex constraint. 
\begin{definition}[Kernel Embedding]\label{def:visitation}
		For a  MDP with kernel feature mapping $ \{\psi_h: \cS \times \cA \rightarrow \RR^d\}_{h\in[H]} $, we define its kernel embedding as \#\label{eq:define-visit}
		{\Psi}^\pi = \big({\Psi}^\pi_h=\EE_{\pi}[\psi_h(s_h,a_h)]\big)_{h\in[H]},
		\#
		where the expectation is taken under the trajectory $\{(s_h,a_h)\}_{h\in[H]}$ induced by policy $\pi$ and the underlying transition.
\end{definition}
The kernel embedding in \eqref{eq:define-visit} represents the agent's visitation distribution on every state-action pair under the policy $ \pi $.  The kernel method is also frequently used in existing MDP literature, as it can be used to incorporate function approximation when designing learning targets. For example, The reward function is often regarded as a linear function of a kernelized feature mapping in RL literature \citep{yang2020provably, jin2020provably,ding2021provably,wu2021offline}. When the reward function is known, we can also take the reward function as the kernel feature \citep{kakade2020information, uehara2022representation}. In an MDP with an underlying kernelized structure, we can evaluate the agent's policy by its initial state value function $V^\pi_1 = \EE_{\pi}[\sum_{h=1}^{H}c_h(s_h,a_h)] $ . The function $V_1^\pi$ can be reformulated to $ V_1^\pi=\theta\cdot{\Psi^\pi}$  under the linear function approximation case \citep{yang2020provably, jin2020provably}, which is a linear mapping with respect to the kernel embedding. {\red{In the general case, when $\psi_h$ is not given, we can learn it through a supervised learning oracle or a model-free exploration, and subsequently employ it in our downstream algorithms, e.g. see Algorithm 1 in \cite{modi2022modelfree}.}} Thus, in an MDP-related optimization, it is reasonable to use the kernel embedding as a measure of how a state-action pair $ (s,a) $ contributes to the objective. To this end, we aim to solve the following optimization problem defined as a constrained convex MDP,
	\begin{equation}\label{eq:convex-prob}
	\min_{\pi\in\Delta(\+A\given\+S,H)} 
	f\big({\Psi}^\pi\big)
	\quad
	\text{s.t.}
	\quad
	g({\Psi}^\pi) \leq 0.
	\end{equation}
% that is, the agent aims to minimize a cost function $f$ determined by the kernel embedding of a certain distribution on $\cS\times\cA$ with horizon $[H]$, while ensuring another constraint $ g $ on the kernel embedding stays negative, which is not necessarily linear. \\
To measure the efficiency of policies in the first $T$ episodes, we introduce the following performance measures,
\begin{equation}\label{eq:perf-metric}
\operatorname{Regret}(T)= T\big(f\big({\Psi}^{\widehat{\pi}}\big)-f\big({\Psi}^{\pi^*}\big)\big),
\quad
\operatorname{Violation}(T)=Tg\big({\Psi}^{\hat{\pi}}\big).
\end{equation}
Here $ {\Psi}^{\widehat{\pi}} = 1/T \sum_{t=1}^{T} {\Psi}^{\pi_k} $ is the average kernel embedding corresponding to the mixed policy $ \widehat{\pi} = \{\widehat{\pi}_h\}_{h=1}^{H} $ of the first episode. By mixed policy we mean the agent rolls out and performs a random policy of index from $1$ to $ T $ in equal probability at the beginning state. The performance measures in \eqref{eq:perf-metric} are widely adopted by previous works in RL where a convex objective function is concerned (see, e.g., \cite{ding2021provably,brantley2021constrained,yu2021provably,wu2021offline}).\\
We remark that our model is more general than the standard RL problem. To see this, we can reduce the C$^2$MDP to standard RL by setting $\psi_h(s_h,a_h) = c_h(s_h,a_h)$ , $f$ as the linear mapping with a $H$-dimension one hot feature vector, and removing the constraint.
% 	\begin{example}[Constrained MDP, \citep{ding2021provably,efroni2020exploration}]
% 	A Constrained MDP target is defined by the following constained optimization,
%     \begin{equation}\label{eq:cmdp}
%     \max_{\pi \in \Delta(\mathcal{A} \mid \mathcal{S}, H)} \EE_{\pi}[\sum_{h=1}^H r_h(s_h,a_h)]
%     \quad
% 	\text{s.t.}
% 	\quad
% 	\EE_{\pi}\bigg[\sum_{h=1}^H g_h(s_h,a_h)\bigg]\leq b,
%     \end{equation}
%     Where $r_h(s,a)$  and $g_h(s,a)$ are the reward and cost function of horizon $h$, and $b$ is a constant. If we focus on linear kernel MDP as in \citep{ding2021provably,ayoub2020model,cai2020provably}, i.e., we assume ${r_h},{g_h}$ to be linear function of a given kernel,$$
%     r_h(s,a) = \psi_h(s,a)\cdot \theta_{r,h},\qquad g_h(s,a) = \psi_h(s,a)\cdot \theta_{g,h}, 
%     $$
%     we have \ref{eq:cmdp} to be a convex MDP with linear objective and rewards,
%     \begin{equation}
%         \max_{\pi \in \Delta(\mathcal{A} \mid \mathcal{S}, H)} \Psi^\pi \cdot \iota_{r} \quad \text{s.t.} \quad \Psi^\pi \cdot \iota_{g}  \leq 0,
%     \end{equation}
%     where $\iota_{r} = (\theta_{r,h})_{h\in[H]}$ and $\iota_{g} = (\theta_{g,h})_{h\in[H]}$.
% 	\end{example}
	\begin{example}[Multi-objective MDP, \citep{yu2021provably,wu2021offline}]
	A Multi-objective MDP considers the following problem, 
	\begin{equation}
	\min_{\pi \in \Delta(\mathcal{A} \mid \mathcal{S}, H)} h_1(\bm{V}^\pi) \qquad \text{s.t.} \qquad h_2(\bm{V}^\pi) \leq 0,
	\end{equation}
	where $$
	\bm{V}^\pi= \bigg(\EE_{\pi}\bigg[\sum_{h=1}^H c_h^i(s_h,a_h)\bigg]\bigg)_{i\in[I]}
	$$
	is the initial state value function vector, and $h_1,h_2: \RR^I \rightarrow \RR$ are $1$-Lipschitz convex functions. If we use linear function approximation {\red{for the cost function}}, i.e.
	$$
	c_h^i(s,a) = \psi(s,a) \cdot \theta_h^i,\quad \forall i \in [I]
	$$
	the Multi-objective MDP turns into a constrained convex MDP,
	$$
	\min_{\pi \in \Delta(\mathcal{A} \mid \mathcal{S}, H)} h_1\big(\Xi\cdot \Psi^\pi\big) \qquad \text{s.t.}\qquad h_2\big(\Xi\cdot \Psi^\pi\big) \leq 0.
	$$
	Here $\Xi = (\theta_1^{1,\top},\cdots \theta_H^{1,\top};\cdots \theta_1^{I,\top},\cdots,\theta_H^{I,\top})$ is a matrix formed by concatenating by $\{\theta_h^i\}_{i\in [I]}^{h\in [H]}$. Note that when $I=1$ and $h_1(x)= h_2(s) = x$, Multi-objective MDP reduces to the constrained MDP in \cite{ding2021provably} and \cite{efroni2020exploration}. We also claim that our model is more general than the one in \citep{yu2021provably,wu2021offline}, since they assume $h_1, h_2$ to be monotone in all components and $h_2$ can only take the form $d(x,\cW)$, with $\cW$ being a convex set.
	\end{example}
	\begin{example}[Feasibility/Apprenticeship Learning, \citep{abbeel2004apprenticeship,syed2008apprenticeship,miryoosefi2019reinforcement,zahavy2020apprenticeship}]
	Feasibility learning considers minimizing the distance between the kernel embedding of the probability induced by the performance policy and a convex set $\cW$, i.e.,
	\begin{equation}
	    \min_{\pi \in \Delta(\mathcal{A} \mid \mathcal{S}, H)} \operatorname{dist}(\Psi^\pi, \cW).
	\end{equation}
	Here $\operatorname{dist}$ can be chosen as any sort of discrepancy measure.When $\cW$ reduces to a singleton $\{\Psi= (\EE_{\cP}[\psi(s_h,a_h)])_{h\in [H]}\} $, i.e.  the kernel embedding of a given probability distribution $\{\cP_h\}_{h\in[H]}$, the optimization reduced to apprenticeship learning. 
\end{example}

\subsection{Examples of the Underlying Transition Models}
Recall that C$^2$MDP is defined for any decision problem with a given linear kernel in its objective. With additional assumptions on the underlying transition, we can define different algorithms for solving it. The transition models we discuss here are  (1) {Kernelized Nonlinear Regulator (KNR)} setting and (2)  {Low-rank MDP} setting, which {\red{cannot}} be solved by algorithms design for tabular setting.
\vskip4pt
\noindent{\bf Kernelized Nonlinear Regulator.} The \textit{Kernelized Nonlinear Regulator} setting generalizes the linear quadratic regulator (LQR) setting \citep{kakade2020information} and is especially helpful in continuous control problems. A KNR is an MDP with the following transition model,
\#\label{eq:transition}
	s_{h+1}=W^{\star} \phi(s_h, a_h)+\epsilon,\quad \epsilon \sim \mathcal{N}\left(0, \sigma^{2} \cI\right),
\#
for all $h \in [H]$, where $ \phi: \cS\times\cA \rightarrow \cH $ is a given kernel feature mapping of a $d$-dimension space $\RR^d$ \citep{kakade2020information,mania2020active,song2021pcmlp}. The transition parameterization $ W^*$ characterizes the mapping from the feature $ \phi(s_h,a_h) $ to the expectation of the next state $ s_{h+1} $. We also remark that the KNR is a general model in the sense that both the state space $\cS$ and the action space $\cA$ can be continuous.

\vskip4pt
\noindent{\bf Low-rank MDP.} In a Low-rank MDP\citep{uehara2021pessimistic,agarwal2020flambe,modi2022modelfree}, the underlying transition takes the form
\#\label{eq：transition_linear}
\cP_{h}^*\left(s_{h+1}\mid s_h, a_h\right)=\left\langle\phi_h^*(s_h,a_h), \mu^*_{h}\left(s_{h+1}\right)\right\rangle,
\#
for all $ h\in[H] $. Here the vector $\mu_h=(\mu_h^{(1)}, \ldots, \mu_h^{(d)})$ is the concatenation of $d$ unknown (signed) measures over $\mathcal{S}$. Unlike KNR ,  both the feature mapping $ \phi_h^*$ and the measure $\mu_h^*$ in Low-rank MDP are unknown to the agent and need to be learned.  For Low-rank MDP, it is natural to assume the agent access to two function classes $\Theta \subset \cS\times\cA \rightarrow \RR^d$ and {\red{$\Upsilon \subset \cS \rightarrow \RR^d$}} for candidate mappings for learning the true embeddings $(\mu_h^*, \phi_h^*)$. Thus we make the following assumption,
\begin{assumption}[Realizability]\label{ass:realizability}
The model class $(\Theta,\Upsilon)$ with $ \{\mu_h^*\}_{h\in[H]}\subset \Theta$ and $ \{\phi_h^*\}_{h\in[H]} \subset \Upsilon$ is known, where both $\Theta$ and $\Upsilon$ are finite sets.
\end{assumption}

\cite{uehara2022representation} show that the case of finite function class can be easily generalized to infinite case. When feature $\phi^*
$ is known, such a setting degerates to the linear MDP \cite{yang2019sample,yang2020reinforcement} Without loss of generality, we also make the following standard assumptions  \citep{kakade2020information, uehara2022representation}. The choice of the upper bound will not add complexity to our analysis.
\begin{assumption}\label{ass:bound-feature}
We have the following assumptions.
\begin{enumerate}
\item For the KNR case, we assume that the feature $\phi$ of the underlying RKHS is uniformly bounded, i.e., $\|\phi(s,a)\|_2\leq 1$ for all $(s,a) \in \cS\times\cA$. For simplicity, we also assume that the transition parametrization satisfies $\|W^*\|_2\leq 1$, here $\|\cdot\|_2$ is the matrix $2$-norm.
\item For the Low-rank MDPs, we assume that $\|\phi_h(s,a)\|_2\leq1$ for all $(s,a)\in \cS\times\cA$, and for any function $g:\cS \rightarrow [0,1]$ and $\mu \in \Upsilon$, $\|\int_\cS\mu_h(s)g(s)d\nu\|_2 \leq \sqrt{d}$, here $\nu(\cdot)$ is a given abstract measure defined on the state space $\cS$. 
\item For the kernel vectors $\{\psi_h\}_{h\in[H]}$ in the objective and constraint, we assume  $\|\psi_h(s,a)\|_2\leq B$.
\item\label{ass:convex-lipschitz}
We assume that the objective $f$ and the constraint $g$ in \eqref{eq:convex-prob} are convex and $1$-Lipschitz,  which further implies that $\{\|\partial f\|_2,\|\partial g\|_2 \} \leq 1 $.
\end{enumerate}
\end{assumption}

For both cases mentioned above, the underlying transition probability is unknown, and can only be estimated through stochastic interactions with the environment. Thus, directly representing the set of all kernel embedding, i.e., $  \cV = \{{\Psi}^\pi:\text{any }\pi \in \Delta(\+A\given\+S,H)\}  $ is impossible, which makes \eqref{eq:convex-prob} a challenging problem. As a consequence, we cannot simply regard \eqref{eq:convex-prob} as a constrained optimization problem. Instead, we have to learn the optimal policy by collecting data via interacting the environment. Moreover, with the general constraint $ g({\Psi}^\pi)\leq 0$ on the distribution, the simple dual optimization method for set constraint  \citep{yu2021provably} becomes infeasible. To address these challenges, we introduce a primal-dual algorithm in the subsequent section.

\section{Main Algorithm}
In this section, we provide a primal-dual algorithm \underline{V}ariational \underline{P}rimal-\underline{D}ual \underline{P}olicy \underline{O}ptimization (VPDPO) for \eqref{eq:convex-prob}, which achieves sublinear in both regret and constraint violation. \\
\subsection{Reformulation: Double Duality}
In this subsection, we reformulate  \eqref{eq:convex-prob} as an unconstrained convex-concave problem, so that we can utilize the standard MDP method to solve it. Doing so will enable us to design a provably efficient algorithm.\\
The convex problem \eqref{eq:convex-prob} is nontrivial only  when its feasible set is none-empty. With the set of all reachable kernel embedding $ \cV $, we assume that $\cV \cap \{g({\Psi}^\pi)\leq0\} $ is not empty, so that \eqref{eq:convex-prob} is well-posed. 
% Here we extend the definition of $ \pi $ from simply Markov policy to the set of mixed policy (\cite{altman1999constrained}). 
To verify the convexity of feasible set \eqref{eq:convex-prob}, we first present the following proposition.
\begin{proposition}[Convex Problem]\label{prop:convexity}
The generalized optimization problem in \eqref{eq:convex-prob} is a convex problem.
\end{proposition} 
\begin{proof}
    See Appendix \ref{appendix-cvx-proof} for detailed proof.
\end{proof}

 Next, we make the following assumption on $g$, which is standard in convex optimization and constrained convex MDP literature (\cite{zahavy2021reward}, \cite{efroni2020exploration}, \cite{ding2021provably}). 
\begin{assumption}[Slater Point]\label{ass:slater}
There exists a policy $ \pi^{\prime}$, such that  \eqref{eq:convex-prob} holds with strict inequality, i.e., $g({\Psi}^{\pi^{\prime}})<0 $.
\end{assumption}
{\red{Note that in Assumption \ref{ass:slater}, we do not require a pre-knowledge for $\pi'$. From an optimization perspective, a problem-dependent Slater condition is a measure of the size of the feasible region and
determines the difficulty of solving a constrained optimization. The absence of such a condition may result
in the lack of constraint qualification and cause failure in even simple optimization problems, for example,
see \cite{hijazi2016constraint}.}} With Assumption \ref{ass:slater}, we can reformulate \eqref{eq:convex-prob} to a standard Lagrangian optimization problem (Corollary 28.1.1, \cite{rockafellar1970convex}). The Lagrangian function of \eqref{eq:convex-prob} takes the form
	\#\label{eq:lagrangian}
	\min_{\Psi \in \cV}\max_{\gamma \geq 0}\big(f({\Psi})+\gamma\cdot g({\Psi})\big).
	\#
	
% 	When the Slater's condition holds, \eqref{eq:lagrangian} and \eqref{eq:convex-prob} share the same solution , so they are equivalent problems.
	
	Slater's condition not only justifies the application of the Lagrangian duality but also allows us to bound the optimal value of the Lagrangian dual variable $ \gamma^* $ from above, which will further be helpful for our algorithm for the gradient update of the dual variables.
	\begin{lemma}[Bounded Lagrangian Dual Variable]\label{lem:bound-mu}
		With Slater's condition in \eqref{ass:slater} , we have \#\label{eq:bound-mu}
		0 \leq \gamma^* \leq \Gamma := -\big(f({\Psi^{\pi^\prime}})-f({\Psi^{\pi^*}})\big)/g({\Psi^{\pi^\prime}}).
		\#
  
	\end{lemma}
\begin{proof}
	See Appendix \ref{sec:optimization} for detailed proof.
\end{proof}
Lemma \ref{lem:bound-mu} provides an upper bound for the optimal dual variable $ \gamma^* $. In order to find $\gamma^*$, we only need to focus on the interval $[0,\Gamma]$.  {\red{In practice, we only need to know an upper bound of $\Gamma$, which can be easily achieved through linear search.}}

Since $f$, $ g $ are $1$-Lipschitz continuous and satisfy the closed-proper function condition, we have 
	\#
	f({\Psi}^{\pi}) = \max_{{\alpha} \in \cB^{dH}} \big({\alpha}^\top{\red{\Psi^\pi}}-f^*({\alpha})\big) , \quad 
	\gamma g({\Psi}^{\pi}) = \max_{{\beta}/\gamma \in \cB^{dH}}\big({\beta}^\top{\Psi^\pi}- \gamma\cdot g^*({\beta}/\gamma)\big),
	\#
	for all $ \gamma \geq 0 $ (Corollary 13.3.3, \cite{rockafellar1970convex}). Here $ f^* $ and $ g^* $ are the Fenchel duals of $ f $ and $ g $, respectively. With these relations, we linearize the objective functions in \eqref{eq:lagrangian} by introducing the variables $ {\alpha} $, $ {\beta} $,
	\begin{equation}\label{eq:minimax-prob}
		\min_{{\Psi}^\pi \in \cV}\max_{\gamma\geq0,{\alpha},{\beta}/\gamma \in \cB^{dH} } \cD({\alpha},{\beta},\gamma,\pi)= \bigl(({\alpha}+{\beta})^\top{\Psi}^\pi-f^*({\alpha})-\gamma\cdot g^*({\beta}/\gamma)\bigr).
	\end{equation}
	We now reformulate the originally non-linear minimization problem into a min-max problem that is linear in $ {\Psi} $ and concave in $({\alpha},{\beta},\gamma)$. Note that $\cV$ is a closed convex set due to our setting and Assumption \ref{ass:bound-feature}. Meanwhile, the feasible set for the dual variables $({\alpha},{\beta},\gamma)$ is a convex compact set. Therefore, by the minimax theorem \citep{rockafellar1970convex}, we can reformulate \eqref{eq:minimax-prob} to
	\begin{equation}\label{eq:maxmin-prob}
		\max_{\gamma\geq0,{\alpha},{\beta}/\gamma \in \cB^{dH} }\min_{{\Psi}^\pi \in \cV} \cD({\alpha},{\beta},\gamma,\pi)= \big(({\alpha}+{\beta})^\top{\Psi}^\pi-f^*({\alpha})-\gamma\cdot g^*({\beta}/\gamma)\big).
	\end{equation}
	In the rest of this paper, we denote by $ {\alpha^*} $, $ {\beta^*} $ and $ \gamma^* $ the optima of the dual variables in \eqref{eq:maxmin-prob}, $ \pi^*=\{\pi_h^*\}_{h\in[H]} $ the optimal policy, and $ {\Psi}^* $ the kernel embedding corresponding to $\pi^*$. 
We can rewrite \eqref{eq:maxmin-prob} as $\max_{\gamma,{\alpha},{\beta}}\cL(\gamma,{\alpha},{\beta})$, where $\cL(\gamma,{\alpha},{\beta}) = \min_{{\Psi}^\pi \in \cV}\cD({\alpha},{\beta},\gamma,\pi)$. When the dual variables are fixed, it suffices to implement model-based value iteration for solving $\cL{(\gamma,{\alpha},{\beta})}$. By simultaneously updating $\gamma$, ${\alpha}$ and ${\beta}$, we can reach optimality by a primal-dual method.

	\begin{remark}
	In \eqref{eq:maxmin-prob}, the term $\gamma g^*({\beta}/\gamma)$ is a convex function composed with a perspective function, so it must be convex in $({\beta},\gamma)$. See \cite{boyd2004convex} for details.
	\end{remark}
\subsection{Solution: Primal-Dual Method}
The minimax structure in \eqref{eq:minimax-prob} implies us to implement a primal-dual method. Such implementation is common when facing nonlinearity in visitation measures (e.g., \cite{wu2021offline} and \cite{efroni2020exploration}).  \\
\textbf{Dual Update.} We perform an online projected gradient ascent method for a dual update. In each iteration, we update ${\alpha} $ by moving ${\alpha}^k$ to a direction of maximizing the dual function $\cD({\alpha},{\beta},\gamma,\pi)$ and then project it to the unit ball. To represent the projection set for $ ({\beta},\gamma) $, we combine the restriction imposed by Fenchel dual and Slater's condition and define\
\begin{align*}
\+G = \{({\beta},\gamma):\|{\beta}\|_2\leq\gamma, \gamma\in[0,\Gamma]\}.
\end{align*}
When the Slater's condition holds, the optimal solution $ ({\beta}^*,\gamma^*) $ always lies in $ \+G $ by Lemma \ref{lem:bound-mu}. If we know the underlying transition map $ W^* $ in priori, we can solve the outer iteration of the minimax problem in \eqref{eq:minimax-prob} by value iteration and implement $ {\Psi}^{\pi^t} $ in the gradient ascent step. However, since the transition remains obscure to us, we use $ {\Psi}^t = ({\Psi}_h)_{h\in[H]} $ as a proxy, where $ {\Psi}^t_h = \EE_{\pi,\cP^t}[\Psi_h(s_h,a_h)]  $. In the dual update, the step size $\eta^t$ is set as $\cO(1/\sqrt{t})$ (or $\cO(1/\sqrt{T})$ when $T$ is given). In Algorithm \ref{alg:OPDO}, $\partial_\gamma$ and $\partial_{{\beta}}$ are the subgradient operator with $\gamma$ and ${\beta}$ as the variable, respectively.
\begin{algorithm}[H]
   \caption{Variational Primal-Dual Policy Optimization}
   \begin{algorithmic}[1]\label{alg:OPDO}
   	 \REQUIRE  Step size $ \{\eta^t\}_{t=1}^{T} $, $ {\alpha}^1 \in \cB^{dH} $, $ \gamma^1 \in [0,\Gamma] $, $ {\beta}^1 \in \gamma^1 \cdot\+B^{dH}$
   	 \FOR{$t=1, \dots,T$}
   	 \STATE $ {\alpha}^{t+1} \leftarrow\Pi_{\+{B}^{dH}}\big\{{\alpha}^{t}+\eta^{t}\big({{\Psi}}^t-\partial f^*(\alpha^t)\big)\big\}$
   	 \STATE $ 		\hat{{\beta}}^{t+1} \leftarrow{\beta}^{t}+\eta^{t}\big({\Psi}^t-\partial_{{\beta}} g^{*}({\beta}^{t}/\gamma^t)\big) $
   	 \STATE $ \hat{\gamma}^{t+1} \leftarrow \gamma^t + \eta^t\big(\partial_\gamma (-\gamma^tg({\beta}^t/\gamma^t))\big) $
   	 \STATE $ ({\beta}^t, \gamma^t)= \Pi_{\+G}(\hat{{\beta}}^{t+1},\hat{\gamma}^{t+1}) $
   	 \STATE $ {\theta}^{t+1} \leftarrow {\alpha}^{t+1} +{\beta}^{t+1} $\label{line:update}
   	 \STATE Update the cost function $ \{c_h^{t+1}(s,a)= \theta^{t+1}_h\cdot \psi(s,a)\}_{h\in[H]}  $ 
   	 \STATE Update the confidence set $\cC^{t+1}$ by Algorithm \ref{alg:OE} or \ref{alg:OE_low_rank} 
   	\STATE $ (\pi^{t+1},\cP^{t+1}) \leftarrow \text{argmin}_{\pi} \min_{\cP\in\cC^{t+1}} {\red{V_{1,\cP}^{t+1,\pi}}} $. \label{line:optimistic-update}
   	 \STATE Calculate $\Psi^{t+1} = (\EE_{\pi^{t+1},\cP^{t+1}}[\psi(s_h,a_h)])_{h\in[H]}$
   	 \ENDFOR 
   \end{algorithmic}
\end{algorithm}

\noindent\textbf{Primal Update: Construct a cost.} Algorithm \ref{alg:OPDO}  further relies on the agent's exploration to estimate the transition $ \cP^t $ with experience in the previous $ t-1 $ episodes.  Since an explicit cost does not necessarily occur in our optimization problem, to implement value iteration, we construct a cost by introducing the dual vector $ {\theta}^t = {\alpha}^t+{\beta}^t$ for all $t$, and set a temporary reward $c^t_h = \psi_h\cdot \theta_h$. Note that in the minimax problem \eqref{eq:minimax-prob}, with fixed $({\alpha},{ \beta})$, the objective function turns into
\begin{align*}
\min_\pi\big(({\alpha}+{\beta})\cdot{\Psi}^\pi \big)= \sum_{h=1}^{H} \EE_{\pi,\cP^*}[\Psi_h(s_h,a_h)\cdot({\alpha}_h+{\beta}_h)],
\end{align*}
which can be viewed as an accumulative cost minimization problem. This is essentially an optimal control problem.
 Corresponding to $c_h^t$, we set the value functions
\#\label{eq:value-set}
{\red{V_{h,\cP}^{t,\pi}(s)}} & = \EE_{\pi,\cP}\bigg[\sum_{i=h}^{H}c_h^t(s_i,a_i)\,\bigg|\, s_h=s\bigg],\\
{\red{Q_{h,\cP}^{t,\pi}(s,a)}} & = \EE_{\pi,\cP}\bigg[\sum_{i=h}^{H}c_h^t(s_i,a_i)\,\bigg|\, s_h=s,a_h=a\bigg].
\#
for policy $\pi$. {\red{For simplicity, we denote $V_{h,\cP}^{t,\pi_t}$ and ${\red{Q_{h,\cP}^{t,\pi_t}}}$ as $V^{t}_{h,\cP}$ and $Q^t_{h,\cP}$, respectively}}. Here and in the rest of this paper, we denote by $\EE_{\pi,P}[\cdot]$ the expectation taken over the trajectory $ \{(s_h,a_h)\}_{h\in[H]} $ induced by $\{\pi_h\}_{h\in[H]}$ and the underlying transition kernel $\cP$. With the confidence set $\cC^t$ given by Algorithms \ref{alg:OE} and \ref{alg:OE_low_rank}, Line \ref{line:optimistic-update} in Algorithm \ref{alg:OPDO} follows the principle of ``\textit{Optimism in the Face of Uncertainty}", and chooses the policy and model in the confidence set that can incur the smallest cost.  {\red{We highlight that Algorithm \ref{alg:OPDO}} is a model-based algorithm, as it explicitly learns the underlying transition probability.}
\begin{algorithm}[H]
	\caption{VPDPO for KNR case}
     \begin{algorithmic}[1]\label{alg:OE}
     	\REQUIRE $ \{(s_h^i,a_h^i)\}_{i\in[t],h\in[H]} $  , $\lambda>0$, $C_1 > 0$, $\theta^t$,$\Lambda_0 =\lambda I$, $\pi_0 = {a_0}$
     	\STATE Execute $ \pi^{t} $ to sample a new trajectory $ \{(s_h^{t}, a_h^{t})\}_{h\in[H]} $
     	\STATE $ \widehat{W}^{t}\leftarrow \arg \min _{W} \sum_{\tau=1}^{t} \sum_{h=1}^{H}\left\|W \phi\left(s_{h}^{\tau}, a_{h}^{\tau}\right)-s_{h+1}^{\tau}\right\|_{2}^{2}+\lambda\|W\|_{F}^{2} $.
     	\STATE $ \Lambda^{t}\leftarrow \lambda I+\sum_{\tau=1}^{t} \sum_{h=1}^{H} \phi\left(s_{h}^{\tau}, a_{h}^{\tau}\right)\phi\left(s_{h}^{\tau}, a_{h}^{\tau}\right)^{\top} $.
     	\STATE Update $ \cC^t \leftarrow\big\{\cP \mid \big\|\left(W-\hat{W}^{t}\right)\left(\Lambda^{t}\right)^{1 / 2}\big\|_{2}^{2} \leq R^{t},\|W\|_2\leq 1, \cP \text{ parametrized by $W$}\big\} $ with $ R^t $ defined in \eqref{eq:ball-radius}.
     	\RETURN  Confidence set $\cC^t$
     \end{algorithmic}
\end{algorithm}
\begin{algorithm}[H]
\caption{VPDPO for Low-rank MDP case}
    \begin{algorithmic}[1]\label{alg:OE_low_rank}
        \REQUIRE model set $\cM = \{(\mu,\phi):\mu \in \Upsilon,\phi \in \Theta \}$, $\cD_{0,h} =\varnothing$, $\pi_0 = U(\cA)$
        \STATE Collect a set of tuples $\{ (s_h^t,a_h^t,s_{h+1}^t)\}_{h\in[H-1]} $ by rolling out $s_h^t$ with policy $\pi_t$ and then select $a_h^t$ by a uniform distribution on $\cA$, i.e. $a_h^t \sim U(\cA)$, $s_{h+1}^t\sim \cP(\cdot\,|\,s_h^t,a_h^t), $ .
        \STATE Update $\cD_{h,t} = \cD_{h,t-1}\cup \{(s_h^t,a_h^t,s_{h+1}^{t})\}$.
        \STATE $(\hat{\mu}^t_h,\hat{\phi}^t_h) \leftarrow \text{argmax}_{(\mu,\phi)\in \cM}\EE_{\cD_{h,t}} [\log(\mu(s_{h+1})^{\top}  \phi(s_h,a_h))]$ .\label{line:mle-est}
        % \STATE $\Lambda^t = \sum_{\tau =1}^{t}\sum_{h=1}^{H} \hat{\phi}^t(s_h^\tau)\hat{\phi}^t(s_h^\tau)^\top + \lambda^tI$.
        \STATE $\hat{\cP}^t_h(\cdot\,|\,s_h,a_h) \leftarrow \hat{\mu}^t_h(\cdot)^\top\hat{\phi}^t_h(s_h,a_h)$.        
        \STATE Update $\cC^t \leftarrow \big\{\cP = \{\cP_h\}_{h\in [H]} \big| \EE_{\cD_h^t}\big[\|\hat{\cP}_h^t(\cdot \,|\, s_h, a_h)-P_h(\cdot \,|\, s_h, a_h)\|_1^{2}\big] \leq R^t \big\}$ with $R^t$ defined in \eqref{eq:ball-radius-low-rank}
        % \STATE $b^t(s,a) \leftarrow \min\bigg(\xi^t \sqrt{\hat{\phi}^t(s,a)(\Lambda^t)^{-1}\hat{\phi}^t(s,a)}, 2\bigg)$, $r_h^t(s,a) = \psi_h(s,a)^\top\theta_h^t + b^t(s,a)$.
    	\RETURN  Confidence set $\cC^t$

    \end{algorithmic}
\end{algorithm}

Algorithms \ref{alg:OE} and \ref{alg:OE_low_rank} interact with the environment with policy $ \pi^t = \{\pi_h^t\}_{h\in[H]} $ given by Algorithm $\ref{alg:OPDO}$, and then construct confidence set for possible models. In each episode, we construct a confidence set $ \cC^t $, whose center and weighted radius are designed deliberately.
The center of the confidence set is chosen by the maximum likelihood estimation (MLE), and the weighted radius $R^t$ is chosen so that the real transition mapping $ \cP^* $ lies in $ \cC^t $ for every $ t $ with a high probability. Specifically, in Algorithm \ref{alg:OE} we set  \begin{equation}\label{eq:ball-radius}
    R^t = c(\lambda \sigma^{2}+\sigma^{2}(d+\log(t \operatorname{det}(\Lambda^{t}) / \delta \operatorname{det}(\Lambda^{0})))
\end{equation}
for the KNR case, and in Algorithm \ref{alg:OE_low_rank} for Low-rank MDP we set \begin{align}\label{eq:ball-radius-low-rank}
    R^t = c\log(TH|\Upsilon||\Theta|/\delta)/t
\end{align}
The difference between Algorithms \ref{alg:OE} and \ref{alg:OE_low_rank} is that, in the KNR setting, the agent collects a full trajectory by performing the same policy $\pi_t$, while in the Low-rank MDP setting, for each epoch $t$ and for $h\in[H]$, the agent performs $\pi^t$ for the first $h$ step and then augment the trajectory by a randomly choose an action and then transit to the next state, i.e., $a_h\sim U(\cA)$, $s_{h+1} \sim \cP(\cdot\,|\,s_h,a_h)$. Note that this exploration manner only influences the degree of $H$ in the sampling complexity, and does not affect the sublinear regret and violation.

We remark on the computation efficiency of Algorithms \ref{alg:OPDO}-\ref{alg:OE_low_rank}. For Algorithm \ref{alg:OPDO}, the projection set $ \cG $ for dual variable $ ({\beta}, \gamma) $ can be seen as an intersection of a second-order cone $ \{({x},t):\|{x}\|_2 \leq t\} $ and a half space $ \{(x,t):t\in [0,\Gamma]\} $. Projection to both sets has a closed-form solution. The projection to $ \cG $ can thus be computed via implementing the alternating projection method, which involves a sequence of gradient steps and projection \citep{BREGMAN1967200}. The proxy $ {\Psi}^t $ can be estimated by Monte Carlo method, with $ W^t $ as a known transition. {\red{We would also like to remark that the calculation of Line \ref{line:optimistic-update} of Algorithm \ref{alg:OPDO}, known as the \textit{optimistic planning}, is in general NP-hard \citep{dani2008stochastic},}} and we assume there is an oracle to implement it \citep{kakade2020information,uehara2021pessimistic,jin2021bellman,ayoub2020model}. Then we only focus on the statistical complexity. From that,  we make the following assumption.

\begin{assumption}[Black-box Computation Oracle]\label{ass:black-box}
	We assume that there is an oracle that implements
	Line \ref{line:optimistic-update} of Algorithm \ref{alg:OPDO}.
\end{assumption}
{\red{In practice, several effective heuristics may be available through gradient-based methods such as iLQG \citep{todorov2005generalized}, and CIO \cite{mordatch2012contact}, or sampling-based methods, such as MPPI \citep{williams2015model} and DMDMPC \citep{wagener2019online}.}}

In the Low-rank MDP setting, motivated by the estimation of conditional probability \citep{uehara2022representation,agarwal2020flambe}, we use MLE for estimating the underlying transition. Unlike in the KNR case where the MLE has a closed-form solution, it is hard to find a general closed-form solution for representation learning by MLE. Correspondingly, we need an oracle for efficient MLE computation for Line 1 in Algorithm \ref{alg:OE_low_rank}. 
\begin{assumption}[Maximum-Likelihood Estimation]\label{ass:mle-oracle}
Consider the model class $\mathcal{M}$ and a dataset $\mathcal{D}$ in the form of $\left(s, a, s^{\prime}\right)$, the MLE oracle returns the maximum likelihood estimator, $$(\hat{\mu},\hat{\phi})=\text{argmax}_{(\mu,\phi) \in \mathcal{M}} \mathbb{E}_{\mathcal{D}} \big[\log \big(\mu\left(s^{\prime}\right)^{\top} \phi(s, a)\big)\big],$$
which implements Line \ref{line:mle-est} of Algorithm \ref{alg:OE_low_rank}.
\end{assumption}
We assume there exists practical algorithms that avoid explicitly enumerating over all functions in the model space $\cM$. %Such an assumption is common in machine learning literature \citep{du2019provably,misra2020kinematic}. 
In practice, such oracles can be reasonably approximated whenever optimizing over $\cM$ is feasible, such as in neural networks. 
\section{Theoretical Results}
In this section, we provide theoretical analysis for Algorithms \ref{alg:OPDO} and \ref{alg:OE}. For the regret and the constraint violation, we make the decompositions
\begin{align*}
T\big(f({\Psi}^{\hat{\pi}}) -f({\Psi}^*)\big) & = \underbrace{T\big(f({\Psi}^{\hat{\pi}}) -f(\hat{{\Psi}})\big)}_{\displaystyle\text{(R.i)}}
+\underbrace{T\big(f(\widehat{{\Psi}}) -f({\Psi}^*)\big)}_{\displaystyle\text{(R.ii)}},\\
T\big(g({\Psi}^{\hat{\pi}})\big) & = \underbrace{T\big(g({\Psi}^{\hat{\pi}}) -g(\hat{{\Psi}})\big)}_{\displaystyle\text{(V.i)}}
+\underbrace{T\cdot g(\widehat{{\Psi}})}_{\displaystyle\text{(V.ii)}},
\end{align*}
where we recall that ${\Psi}^{\hat{\pi}}=1/T\sum_{t=1}^{T}{\Psi}^{\pi_t}$, $  \hat{{\Psi}}= 1/T\sum_{t=1}^{T}{\Psi}^t $. Here  (R.i)  and  (V.i)  are the estimation errors incurred by the noise in the regression. With the Lipschitz condition imposed on $f$ and $g$, it suffices to bound  $ \|\sum_{t=1}^{T}({\Psi}^{\pi_t}-{\Psi}^t)\|_2 $. We reformulate it into bounding a value difference summation $\sum_{t=1}^T (V_{1,\cP^t}^{t}-V_{1,\cP^*}^{t}) $. 
The gradient update for the dual variables allows us to give an upper bound for (R.ii) and (V.ii) in terms of a value difference sequence as well. \\
However, we first need to handle the non-linearity in \eqref{eq:convex-prob}. By implementing the online gradient ascent method in Algorithm \ref{alg:OPDO}, we can guarantee that the following coupling term can be bounded by the value difference of two processes and an $\cO(\sqrt{T})$ term.
\begin{lemma}[Dual Update: Gradient Ascent]\label{lem:couple-lemma}
	For all $ \gamma \in [0,\Gamma] $, we have
    \#\label{eq:main-bound1}
	T\cdot\big[f(\widehat{{\Psi}})- f({\Psi}^*) + \gamma \cdot g(\widehat{{\Psi}})\big] 
	\leq  \sum_{t=1}^{T} {\theta}^t\cdot({\Psi}^t-{\Psi}^*) + CB\Gamma\sqrt{HT},
	\#
	where $C>0$ is an absolute constant.
\end{lemma}
\begin{proof}
    See Appendix \ref{sec:couple-sec} for detailed proof.
\end{proof}
Lemma \ref{lem:couple-lemma}  displays a coupling between the regret and the constraint violation, which is also frequently met in online algorithms using dual updates, such as CMDP and Multi-objectives \citep{ding2021provably, yu2021provably}. The proof of Lemma \ref{lem:couple-lemma} incorporates the standard regret analysis of online gradient ascent and the self-dual property of Fenchel dual,  which is a common technique in analyzing nonlinear function differences with gradient updates. The occurrence of the coupling term directly comes from the gradient update of the dual variables in Algorithm \ref{alg:OPDO}. \\
In the following lemma we introduce the difference of a sequence of projected kernel embedding, which can be interpreted as the performance difference of two systems in $ T $ episodes.
When $ \cP^* $ falls in $ \cC^t $, by the principle of optimism implemented in Line \ref{line:optimistic-update} of Algorithm \ref{alg:OPDO}, the value difference is always negative. In this paper, we denote the event of $\cP^*\in\cC^t$ for all $t\in[T]$ by $\cE_{cb}$, i.e., $\cE_{cb} = \cup_{t=1}^T\{\cP^*\in\cC^t\}$. By the construction of the confidence set, we can further prove that $ \cP^* $ always lies in $ \cC^t$ with the probability of at least $1-\delta $. With the construction of confidence set in Algorithms \ref{alg:OE} and \ref{alg:OE_low_rank}, we choose the 
transition model and policy that would incur the highest accumulative reward in expectation. Therefore, as long as the real dynamic falls in the confidence set, we can obtain optimism in the sense of the following lemma,
\begin{lemma}[Optimism: Value Difference] \label{lem:value-difference}
	If the real model $\cP^*$ falls in the confidence set $\cC^t$ for all $t$, then we have the following inequality,
	\#
\sum_{t=1}^{T} {{\theta}}^t\cdot({\Psi}^t-{\Psi}^*)\leq 0
	.
	\#
\end{lemma}
\begin{proof}
The inequality comes from the construction of the cost function in \eqref{eq:value-set} and the choice of $\cP^t$ and $\pi^t$ in Line \ref{line:optimistic-update} in Algorithm \ref{alg:OPDO}.
\end{proof}
Conditioning on the event that Lemma \ref{lem:value-difference} holds, we actually claim that the coupling term in \eqref{eq:main-bound1} can be bounded by $O(\sqrt{T})$. Combining this with the optimization trick of Theorem \ref{thm:vio-bound}, we can further prove that (R.i) and (V.i) are bounded by $O(\sqrt{T})$. We leave the detailed proof in Section \ref{sec:couple-sec}.
\begin{lemma}\label{lem:spurious-bound}
	Assume that $\cP^* \in \cC^t$ for all $t\in[T]$. Then for all $ \gamma \in [0,\Gamma] $, we have
	\#
	T(f(\widehat{{\Psi}})- f({\Psi}^*)) & \leq  CB\Gamma\sqrt{HT},\\
	T \cdot g(\widehat{{\Psi}}) & \leq  CB\sqrt{HT}.
	\#

\end{lemma}
We now bound the difference of the coupling of the objective and constraint violation by $\sqrt{T}$, with the estimated feature embedding $\hat{\Psi}$ as a self variable. But what can we say about the difference between the estimated average feature embedding $\hat{\Psi}$ and the real average feature embedding $\Psi^{\hat{\pi}}$, $\|\Psi^{\hat{\pi}} - \hat{\Psi}\|_2$? To tackle this issue, we interpret the difference of the kernel mean embedding as the supreme of a set of value differences. 	For a fixed $ {x} = (x_h)_{h\in[H]} \in \R^{dH}$ with $ \|{x}\|_2 \leq1 $, we can consider $ \sum_{t=1}^{T}({\Psi}^{\widehat{\pi}}- \widehat{{\Psi}})\cdot{x} $ as the value difference of two processes, with cost at stage $h$ defined as $c_h(s_h,a_h) = \psi(s_h,a_h)\cdot x_h$. For simplicity, we denote $x\cdot({\Psi}^{\pi_t}-{\Psi}^t)= V_1^{\pi_t}- V_1^t $. As long as we can uniformly upper bound $V_1^{\pi_t}- V_1^t$ for all $\|x\|_2\leq1$, we can give a bound for $\|\Psi^{\hat{\pi}} - \hat{\Psi}\|_2$. The following lemma allows us to decompose a value difference and is useful in our analysis.
\begin{lemma}[Value Difference Lemma]\label{lem:value-diff-same-r} Consider two MDPs $\big(\mathcal{S}, \mathcal{A}, \{\cP^1_h\}_{h=1}^{H},\{r_h\}_{h=1}^{H} \big)$ and 
$\big(\mathcal{S}, \mathcal{A}, \{\cP_h^2\}_{h=1}^{H}, \{r_h\}_{h=1}^{H}\big)$ and a given policy $\pi = \{\pi_h\}_{h\in [H]}$. Then for all $h\in [H]$ the following relation holds,
\#
V_h^\pi(s) - V^{\pi'}_h(s) = \EE_{\pi,{\red{\cP^{2}}}}\bigg[\sum_{i=h}^{H}(\mathbb{P}_i^1 V_{i+1}^\pi(s_i,a_i) - {\mathbb{P}}_i^{2}V_{i+1}^\pi(s_i,a_i))\,\bigg|\, s_h =s\bigg].
\#
\end{lemma}
\begin{proof}
This lemma is a direct corollary of Lemma \ref{lem:value-difference} in the appendix, as the two MDP share the same reward.
\end{proof}
Next, we directly give the performance guarantees for KNR and low-rank MDP cases, and give a brief proof under this value difference routine for the two cases respectively. Both results contain a $O(\sqrt{T}\log T)$ scale in the regret and violation, which shows that VODPO learns in C$^2$MDPs in a statistically efficient manner. As $T$ grows bigger, the mixed policy $\hat{\pi}$ would achieve an suboptimality that decreases in a $O(\log T/\sqrt{T})$ manner. To the best of our knowledge, this algorithm is the first one that achieves sublinear regret and constraint violation in C$^2$MDP.
\subsection{Analysis of the KNR Case}
% Following \cite{srinivas2009gaussian}, we define the expected maximum information gain as
% $$ \begin{aligned}
% 	\gamma_{T}(\lambda) &=\max _{\mathcal{A}} \mathbb{E}_{\mathcal{A}}\left[\log \left(\operatorname{det}\left(\Lambda^{T}\right) / \operatorname{det}\left(\Lambda^{0}\right)\right)\right] \\
% 	&=\max _{\mathcal{A}} \mathbb{E}_{\mathcal{A}}\bigg[\log \operatorname{det}\bigg(I+\frac{1}{\lambda} \sum_{t=1}^{T} \sum_{h=1}^{H} \phi\left(s_{h}^{t}, a_{h}^{\tau}\right)\phi\left(s_{h}^{t}, a_{h}^{\tau}\right)^{\top}\bigg)\bigg],
% \end{aligned} $$
% where $\cA$ is the set of algorithms that map from the history before episode $ t $ to the next policy $\pi_t$. It can be proved that $\gamma_{T}(\lambda) = \cO\big(d\cdot\log (TH)\big)$ \citep{srinivas2009gaussian}.

\begin{theorem}\label{thm:main-thm}
	Assume that Assumptions \ref{ass:bound-feature}-\ref{ass:slater}  and \ref{ass:black-box} hold. Set $ \lambda=\max\{\sigma^{2} ,1\} $. For Algorithm \ref{alg:OPDO} and \ref{alg:OE}, with probability at least $1-\delta$, the regret is bounded by
	$$
	\operatorname{Regret}(T) \leq \cO\bigg(\Gamma B\sqrt{HT}
	+ CBHd\sqrt{T}\log\bigg(\frac{HT}{d\delta}\bigg)\bigg),$$
	and the constraint violation is bounded by
	$$
	\operatorname{Violation}(T) \leq \cO\bigg(CBHd\sqrt{T}\log\bigg(\frac{HT}{d\delta}\bigg)\bigg).
	$$

\end{theorem}

% For low-rank MDP setting, we have the corresponding results.
% \begin{theorem}\label{thm:main-thm-lowrank}
%     Assume that assumption \ref{ass:mle-oracle} and \ref{ass:model-class} holds. Set $\zeta^t = \log(t|\cM|/\delta) $. For Algorithm \ref{alg:OPDO} and \ref{alg:OE_low_rank}, the regret is bounded by 
% $$
%     \operatorname{Regret}(T) \leq \cO\big(\rho\sqrt{HT} \big)
% $$
% \end{theorem}

\subsection{Proof Sketch of Theorem \ref{thm:main-thm}}
In this section, we sketch the proof of Theorem \ref{thm:main-thm}. The detailed proof is deferred to Appendix \ref{appendix-proof}.

\begin{lemma}[Simulation Lemma]\label{lem:simulation-lemma1}
	For any policy $\pi$, feature mapping $W$, bounded cost $c$, and for any initial state $s_{1}$, with the value function defined in \eqref{eq:value-set}(with a upper bound of $\sqrt{H}$), we have
	$$
	\begin{aligned}
	V^{\pi}_{1,\cP^*}\left(s_{1}\right)-V_{1,\cP}^{\pi}\left(s_{1}\right) 
	\leq \cO \biggl(B\sqrt{H}\cdot\mathbb{E}_{\pi,\cP^*}\bigg[\sum_{h=1}^{H} \big\|\big(W^{\star}-W\big) \phi(s_{h}, a_{h})\big\|_{2}\bigg]\biggr),
	\end{aligned}
	$$
	where the state-value function is defined with underlying cost $ c $. Here $\cP^*$ and $\cP$ are the conditional distribution induced by $W^*$ and $W$, respectively.
\end{lemma}
\begin{proof}
With Lemma \ref{lem:value-diff-same-r} we have  \begin{equation*}
\begin{aligned}
    V_{1,\cP^*}^\pi(s_1)-V_{1,\cP}^\pi(s_1) &\leq \EE_{\pi,\cP^*}\bigg[  B\sqrt{H} \sum_{h=1}^{H} \big\|\cP_h^*(\cdot\,|\,s_h,a_h)-\cP_h(\cdot\,|\,s_h,a_h)\big\|_1\bigg]\\
    % &\leq BH\EE_{\pi,\cP^*} \bigg[ \sum_{h=1}^{H} \big\|\cP_h^*(\cdot|s_h,a_h)-\cP_h(\cdot|s_h,a_h)\big\|_1^2 \bigg]^{1/2}  \\
    &\lesssim B\sqrt{H}\EE_{\pi,\cP^*}\bigg[\sum_{h=1}^H\big\|(W^{\star}-W) \phi(s_{h}, a_{h})\big\|_{2}\bigg],
 \end{aligned}  
\end{equation*}
where the second inequality follows from the estimation
\begin{align*}
\|\cP_h^*(\cdot\,|\,s_h,a_h)-\cP_h(\cdot\,|\,s_h,a_h)\big\|_1 = O\Bigl(\big\|(W^{\star}-W) \phi(s_{h}, a_{h})\big\|_{2}\Bigr)
\end{align*}
from \cite{devroye2018total}. Here we drop the constants that only depend on $\sigma$.
\end{proof}
 By Lemma \ref{lem:simulation-lemma1} and the Elliptical Potential Lemma \citep{uehara2021pessimistic}, following the value decomposition routine, we give an upper bound for the estimation error in terms of the maximum information gain in the following lemma.
\begin{lemma}[Estimation Error]\label{lem:est-error}
	For Algorithms \ref{alg:OPDO} and \ref{alg:OE}, with $ \lambda=\max\{\sigma^2,1\} $, we have $\cP^*\in\cC^t$ for all $t\in[T]$ holds with probability at least $1-\delta$, and
	\#
     T\|{\Psi}^{\widehat{\pi}}- \widehat{{\Psi}}\|_2 \leq CBHd\sqrt{T}\log\bigg(\frac{HT}{d\delta}\bigg)
	\# 
	holds with probability at least $1-\delta$,
 where $C > 0$ is an absolute constant that only depends on $\sigma$.
\end{lemma}
\begin{proof}
    See Appendix \ref{sec:couple-sec} for detailed proof.
\end{proof}
 With the 1-Lipschitz assumption for $f,g$, Lemma \ref{lem:est-error} in fact gives a uniform upper bound for (R.i) and (V.i). Combining the results on regret and constraint violation in Lemma \ref{lem:spurious-bound} with the error estimation in Lemma \ref{lem:est-error}, we finish the proof of Theorem \ref{thm:main-thm}. 

\subsection{Analysis of the Low-rank MDP case}
For the Low-rank MDPs, we also prove the sublinear regret and violation under Algorithms \ref{alg:OPDO} and \ref{alg:OE_low_rank}.
\begin{theorem}\label{thm:main-thm-lowrank}
    Assume that Assumptions \ref{ass:bound-feature}-\ref{ass:slater}  and \ref{ass:mle-oracle} hold. Set $R^t$ as in $ \eqref{eq:ball-radius-low-rank}$. For Algorithms \ref{alg:OE} and \ref{alg:OE_low_rank}, with probability $1-\delta$, the regret is bounded by
    $$
    \operatorname{Regret}(T)\leq O\biggl(\Gamma B\sqrt{HT}+B\sqrt{TH|\cA|d^2}\log\biggl(\frac{TH|\Theta||\Upsilon|}{\delta}\biggr)\biggr),
    $$
    and the constraint violation is bounded by 
    $$
    \operatorname{Violation}(T)\leq O\biggl(B\sqrt{TH|\cA|d^2}\log\biggl(\frac{TH|\Theta||\Upsilon|}{\delta}\biggr)\biggr).
    $$
\end{theorem}
{\red{We remark that our regret and constraint violation guarantees in Theorem \ref{thm:main-thm} and \ref{thm:main-thm-lowrank} serve as \textit{Probably Approximately Correct (PAC)} bounds: with probability at least $1-\delta$, we can obtain a Markov policy $\widehat{\pi}:= \frac{1}{T}\sum_{t=1}^T \pi^t$ such that   $f(\Psi^{\widehat{\pi}}) - f(\Psi^{{\pi}^*}) = O(1/\sqrt{T})$, and the constraint violation $g(\Psi^{\widehat{\pi}}) = O(1/\sqrt{T})$}. Consequently, with a sample complexity of $O(1/\epsilon^2)$, the Markov policy $\widehat{\pi}$ such that $f(\Psi^{\widehat{\pi}}) - f(\Psi^{\pi^*})\leq \epsilon$ and $g(\Psi^{\widehat{\pi}}) \leq \epsilon$ hold simultaneously with hight probability. From an asymptotic perspective, with $T$ tends to infinity, $f(\Psi^{\widehat{\pi}})$ converges to the optimal value, while the violation of constraint $g(\Psi^{\widehat{\pi}})$ can be arbitrarily small with high probability. 
%This conforms with current results in safe reinforcement learning, such as \cite{wu2021offline,yu2021provably,ding2021provably,brantley2021constrained}.  
Our result is different in form from the standard definitions in online convex optimization due the existence of both optimality gap and constraint violation. %\citep{hazan2019introduction}, in which regret is defined by $ \sum_{t\in[T]} f(\Psi^{\pi^t}) -f(\Psi^{\pi^*})$. In fact, how to design an algorithm with a sublinear regret in this form is still unknown even for Multi-objective MDP under a bandit setup, see e.g. \cite{agrawal2014bandits}, and we leave it to future works. 
}
\subsection{Proof Sketch of Theorem \ref{thm:main-thm-lowrank}}
In this section we briefly sketch the proofs of efficiencies of Algorithm \ref{alg:OPDO} and \ref{alg:OE_low_rank} in the Low-rank MDP setting. For detailed proof, see Appendix \ref{appendix:proof-lowrank}. \\
We define the state-action visitation induced by the mixed Markov policy before epoch $t$ and the one augmented by choosing random action, $$\rho^t_h(s_h,a_h) = \frac{1}{t-1} \sum_{i \in[t-1]}d_{\pi^i,h,\cP^*}(s_h,a_h),\quad \hat{\rho}^t_h(s_h,a_h) = \frac{1}{t-1} \sum_{i \in[t-1]}d_{\pi^i,h,\cP^*}(s_h) u(a_h),$$ where $d_{\pi,h,\cP}(s_h,a_h)$ is the visitation probability on the $h$-th state-action pair induced by policy $\pi$ and transition kernel $\cP$, and $u(a)$ is the uniform distribution on the action set $\cA$. By implementing MLE in every epoch $t$, we claim that with high probability, the model error under the distribution of the previous policy $\EE_{\hat{\rho}^t}[\|\hat{\cP}^t_h(\cdot|s_h,a_h)-\cP_h^*(\cdot|s_h,a_h)\|_1^2] $ is of $ \tilde{O}(1/t)$.  With a standard Bernstein-type argument for martingales, we have the following lemma.

\begin{lemma}[Shrinking Confidence Ball]\label{lem:shrink-ball}
With probability at least $1-\delta$, we have $\cP^* \in \cC^t$ and
$$
\EE_{\hat{\rho}^t_h}\big[\|{\cP}_h^*(\cdot \,|\, s_h, a_h)-\cP_h(\cdot \,|\, s_h, a_h)\|_1^{2}\big] \leq  \frac{c\log(TH|\Upsilon||\Theta|/\delta)}{t},
$$
for all transition $\cP \in \cC^t $, $t\in [T]$ and $h \in [H]$, where $c$ is an absolute constant. Here $\hat{\cP}_h^t$ is the transition learned by the MLE in Algorithm \ref{alg:OE_low_rank}.
\end{lemma}
\begin{proof}
See Appendix \ref{appendix:proof-lowrank} for details.
\end{proof}
Lemma \ref{lem:shrink-ball} implies that with high probability, our choice of the confidence set $\cC^t$ is good enough for the real transition to fall in. Moreover, the distance between $\cP^*$ and other elements in $\cC^t$ also decreases under the distribution of the mixed policy $\hat{\rho}^t$. By the construction of the value function of \eqref{eq:value-set}, we obtain the following lemma.
As in the KNR case, we also care for the error brought by our insufficient model estimation, $\|\Psi^t - \Psi^{\pi_t}\|_2$. To overcome this tissue, the underlying linear structure of low-rank MDPs is crucial. We introduce the following lemma, which is a modification of Lemma 16 in \cite{uehara2022representation}:
\begin{lemma}
Take any $h \in \mathcal{S} \times \mathcal{A} \rightarrow \mathbb{R}$ such that $\|h\|_{\infty} \leq D$. Then,
$$
\EE_{\pi}[h(s_h,a_h)] \leq \EE_\pi\|\phi_{h-1}^{\star}(s_{h-1},a_{h-1})\|_{\Sigma_{\rho^t, \phi_{h-1}^\star}^{-1}}\sqrt{t|\mathcal{A}| \EE_{\hat{\rho}^t}[h^{2}(s, a)]+\lambda d D^{2}},
$$
where $\Sigma_{\rho^t, \phi^{\star}_h}=t \mathbb{E}_{(s, a) \sim \rho^{t}}\left[\phi_h^{\star}(s, a) \phi_h^{\star}(s, a)^{\top}\right]+\lambda I$. 
\end{lemma}
\begin{proof}
See Appendix \ref{appendix:proof-lowrank} for details.
\end{proof}
Note that here the parameter $\lambda$ and the matrix $\Sigma_{\rho^t,\phi_h^*}$ do not occur in the actual implementation.  This lemma introduces an elliptical potential structure. By then, using the same method as in the KNR case, we prove the upper bound for the 2-norm estimation error.
\begin{lemma}[Estimation Error]\label{lem:appendix-error-lowrank}
With Assumption \ref{ass:bound-feature} and Algorithm \ref{alg:OE_low_rank}, we have
\begin{align*}
T\|\hat{\Psi} - \Psi^{\hat{\pi}}\|_2 \leq c\sqrt{TH|\cA|d^2}\log\biggl(\frac{TH|\Theta||\Upsilon|}{\delta}\biggr)
\end{align*}
holds with probability at least $1-\delta$.
\end{lemma}
Note that $f,g$ are both $1$-Lipschitz by Assumption \ref{ass:convex-lipschitz}, we can thereby control the upper bound of $|f(\Psi^\pi)-f(\hat{\Psi})|$ and $|g(\Psi^\pi)-g(\hat{\Psi})|$ can be bounded by the same scale, combine this with Lemma \ref{lem:spurious-bound} concludes our proof. 
\subsection{Applications to Concrete Examples}
With the general results above, we also highlight their applications on concrete examples {\red{rise in RL}}. In Section \ref{sec:background} we introduced several settings that are well known in MDP literature% \citep{yu2021provably, wu2021offline, abbeel2004apprenticeship, , zahavy2020apprenticeship}, 
which can be regarded as examples of C$^2$MDP, with $f$ being their objectives and $g$ being their constraints. 
We then implement VODPO to solve them, In this section we use Multi-objective MDPs and Feasiblity Learning as examples to show the power of VODPO. 
First, we have the following corollary for the KNR setting, 
\begin{corollary}
 Under Assumptions \ref{ass:realizability} - \ref{ass:black-box}, we assume that $\|\theta_h^i\|_2 \leq \sqrt{d}$ for all $i\in [I] $ and $h\in[H]$. Set $\lambda = \max\{\sigma^2,1\}$, we have \begin{equation}
     \operatorname{Regret}(T) \leq
     \begin{cases}
      O\big(\Gamma H\sqrt{IT}
	+ CH^{3/2}d\sqrt{IT}\log\big(\frac{HT}{d\delta}\big)\big) & \text{for Multi-objective MDP,}\\
      \cO\big(CHd\sqrt{T}\log\big(\frac{HT}{d\delta}\big)\big) & \text{for Feasibility Learning,}
     \end{cases}
 \end{equation}
 and
 \begin{align*}
 \operatorname{Violation}(T) \leq O\biggl(CH^{3/2}d\sqrt{IT}\log\biggl(\frac{HT}{d\delta}\biggr)\biggr) \text{\qquad for Multi-objective MDP,}
 \end{align*}
 hold with probability at least $1-\delta$. Here $C > 0$ is an absolute constant that only depends on $\sigma$.
\end{corollary}
Under the low-rank MDP setting, we have similar results.
\begin{corollary}
  Under Assumptions \ref{ass:realizability} - \ref{ass:mle-oracle}, assuming that $\|\theta_h^i\|_2 \leq \sqrt{d}$ for all $i\in [I] $ and $h\in[H]$, we have
  \begin{equation}
      \operatorname{Regret}(T) \leq \begin{cases}
       O\bigl(\Gamma H \sqrt{IdT} + H\sqrt{|\cA|Id^3T} \log\bigl(\frac{TH|\cM|}{\delta}\bigr)\bigr) & \text{for Multi-objective MDP,}\\
       O\bigl(\sqrt{H|\cA|d^2T}\log\bigl(\frac{TH|\cM|}{\delta}\bigr)\bigr) & \text{for Feasibility Learning,}
      \end{cases}
  \end{equation}
  and
  \begin{align*}
     \operatorname{Violation}(T) \leq O\biggl(H\sqrt{|\cA|Id^3T} \log\biggl(\frac{TH|\cM|}{\delta}\biggr)\biggr)\qquad \text{for Multi-objective MDP,}
  \end{align*}
  hold with probability at least $1-\delta$. Here $|\cM| = |\Theta||\Upsilon|$ is total number of the model classes.
\end{corollary}
{\red{We claim that when $f$ degenerates to a linear function, our results recover the regret of standard KNR  in \cite{kakade2020information}. Specifically, our results in regret matches Theorem 3.2\cite{kakade2020information} in terms of $H$ and $d$, where they accomplish a regret of $O(\sqrt{H^3d^2T})$. When considering a low-rank MDP with a finite horizon, \cite{uehara2022representation} achieves a regret of $O(\sqrt{d^3 H^2 |\cA| T})$, which is also consistent with our result for low-rank MDP case. We also compare our results of low-rank MDP with existing works such as \cite{yu2021provably}}, which focuses on the study of online Multi-objective MDP under the tabular case. Tabular MDP can be regarded as a special case of low-rank MDP with a known feature, with the dimension $d = SA$. By assuming approachability, \cite{yu2021provably} propose an algorithm with regret of $O(\Gamma\sqrt{IH^3S^2A/T})$ and a constraint violation of $O(\sqrt{IH^3S^2A/T})$ , where $S$ and $A$ are the cardinality of $\cS$ and $\cA$, respectively. We claim that our results have a higher-order dependence on $d = SA$ due to the error inherited from MLE and the invoke of one step back inequality. For a technical understanding, we recommend the readers to Appendix \ref{appendix:proof-lowrank}. }
\section{Conclusion}
In this paper, we have developed a provably efficient online algorithm, \underline{V}ariational \underline{P}rimal-\underline{D}ual \underline{P}olicy \underline{O}ptimization (VPDPO) for constrained constrained convex MDP. KNR and Low-rank MDP are two examples. The algorithm extends the reward-based RL algorithm to constrained convex MDP where no explicit reward is needed and incorporates the Lagrangian primal-dual method to transform the constrained optimization into a minimax problem. To handle the balance between exploration and exploitation, we follow the principle of optimism in the face of uncertainty.  We prove that that our algorithm enjoys a $ \tilde{\cO}(\sqrt{T}) $ regret and a $ \tilde{\cO}(\sqrt{T}) $ violation with high probability under standard optimization assumptions, where $ T $ is the total number of episodes taken by the algorithm.

% Acknowledgements and Disclosure of Funding should go at the end, before appendices and references

\acks{Mengdi Wang acknowledges the support by National Science Foundation grants DMS-1953686, IIS-2107304, CMMI1653435, CPS2312093, ONR grant1006977, Google Research and C3.AI. Zhaoran Wang acknowledges National Science Foundation (Awards 2048075, 2008827, 2015568, 1934931), Simons Institute (Theory of Reinforcement Learning), Amazon, J.P. Morgan, and Two Sigma for their support. Zhuoran Yang acknowledges Simons Institute (Theory of Reinforcement Learning) for their support.  We would like to thank Yufeng Zhang, Sirui Zheng, and Runzhe Wu for helpful discussions about parts of this paper. We would also like to thank
the associate editor and the reviewers for many constructive comments that improved the presentation of the paper.}

% Manual newpage inserted to improve layout of sample file - not
% needed in general before appendices/bibliography.

\newpage

\appendix

% Manual newpage inserted to improve layout of sample file - not
% needed in general before appendices/bibliography.

\appendix
\section{Additional Notations}
We write $\P(A)$ as the probability of event $A$. For a KNR, $\P(\cdot\mid W,s_h,a_h)$ denotes the probability distribution over $\cS$ when the agent is in state $s_h$ and takes action $a_h$, with the transition parametrization $W$. For two series $\{a_n\}_{n \geq 1}$ and $\{b_n\}_{n \geq 1}$, we write $a_n \lesssim b_n$ if $a_n \leq C\cdot b_n$ holds for constant $C$ and all sufficient large $n$.
\section{Proof of Lemma \ref{lem:couple-lemma}} \label{sec:couple-sec}
In the dual update, the cost is related to the non-stationary variable ${\theta}$. With the summation of the value difference bounded, we directly prove the following lemma by adding a $\cO(\sqrt{T})$ scale regret which comes from the employment of online gradient ascent.

Recall that $ \widehat{{\Psi}} = 1/T\sum_{t=1}^{T}{\Psi}^t $. We present the following lemma, which can be seen as a corollary of Theorem \ref{thm:proj-subgrad}. 
\begin{lemma}\label{lem:error-bound}
Suppose that Assumptions \ref{ass:bound-feature} and  \ref{ass:convex-lipschitz} holds. For all $ \gamma \in [0,\Gamma] $, we have
	\begin{align*}
	T\big(f(\widehat{{\Psi}})- f({\Psi}^*) + \gamma \cdot g(\widehat{{\Psi}})\big)\lesssim B\Gamma\sqrt{HT},
	\end{align*}
	which further implies
	\begin{align}\label{eq:reg-bound}
	T(f(\widehat{{\Psi}})- f({\Psi}^*)) &\lesssim  B\Gamma\sqrt{HT},\\
T \cdot g(\widehat{{\Psi}})
	&\lesssim  \sqrt{HT}.\label{eq:vio-bound}
	\end{align}
\end{lemma}
\begin{proof}
We have the following relations holds for all $ \gamma \in [0,\Gamma]$,
\#
&f(\widehat{{\Psi}})- f({\Psi}^*) + \gamma \cdot g(\widehat{{\Psi}}) \\
&\qquad=\max_{{\alpha} \in \cB,{\beta} \in \gamma\cdot \cB}\big\{
{\alpha}^\top\widehat{{\Psi}}-f^*({\alpha})- f({\Psi^*}) + {\beta}^\top\widehat{{\Psi}} -\gamma g^*({\beta}/\gamma)\big\}.\notag
\#
Thus, the dual update is equivalent to implementing online gradient ascent on $h_t$, where
\begin{align*}
h_t({\alpha},{\beta},\gamma) = {\alpha}^\top{\Psi}^t-f^*({\alpha}) + {\beta}^\top{\Psi}^t -\gamma g^*({\beta}/\gamma).
\end{align*}
By Theorem \ref{thm:proj-subgrad}, we set the step size $ \eta_t = 2\Gamma/H\sqrt{t} $ (or $ 2\Gamma/H\sqrt{T} $ when $ T $ is pre-decided), the constants $R = 2\Gamma$ and $G=2B\sqrt{H}$ (to verify the conditions, note that $g^*$ is $B\sqrt{H}$-Lipschitz, see \cite{DUBOVITSKII19651}) to get
\begin{equation}\label{eq:bound-lag}
\begin{aligned}
&T\bigg[f(\widehat{{\Psi}})- f({\Psi}^*) + \gamma \cdot g(\widehat{{\Psi}})\bigg] \\
&\qquad \leq \sum_{t=1}^{T}{\Psi}^t\cdot({\alpha}^t+{\beta}^t) - f^*({\alpha}^t) - \gamma^tg^*({\beta}^t/\gamma^t) - Tf({\Psi}^*)+ CB\Gamma\sqrt{HT},
\end{aligned}
\end{equation}
where $ C $ is an absolute constant. With $\gamma^t\geq0$ and $ g({\Psi}^*) \leq0 $, by the definition of Fenchel dual, we have
\#\label{eq:bound-fenchel}
0\geq\gamma^tg({\Psi}^*) \geq {\beta}^t\cdot{\Psi^*} - \gamma^t g^*({\beta}^t/\gamma^t),\quad
 f({\Psi^*}) \geq {\alpha}^t\cdot{\Psi}^*-f({\alpha}^t).
\#
Recall that $ {\theta}^t = {\alpha}^t+{\beta}^t $. Plugging \eqref{eq:bound-fenchel} back to \eqref{eq:bound-lag}, we obtain the following relation holds for all $ \gamma \in [0,\Gamma] $,
\begin{equation}\label{eq:main-bound}
\begin{aligned}
T\bigg[f(\widehat{{\Psi}})- f({\Psi}^*) + \gamma \cdot g(\widehat{{\Psi}})\bigg]
&\leq \sum_{t=1}^{T} {\theta}^t\cdot({\Psi}^t-{\Psi}^*) + cB\Gamma\sqrt{HT}\\ \nonumber
&\leq  cB\Gamma\sqrt{HT}, \nonumber
\end{aligned}
\end{equation}
where the second inequality comes from Lemma \ref{lem:value-difference} and $c$ is an absolute constant. With $\gamma=0$ we obtain \eqref{eq:reg-bound}.
With $\gamma = \Gamma$, we  have $$T\bigg[f(\widehat{{\Psi}})- f({\Psi}^*) + \Gamma \cdot g(\widehat{{\Psi}})\bigg]\lesssim \Gamma\sqrt{HT}.
$$
And with Theorem \ref{thm:vio-bound} we obtain \eqref{eq:vio-bound}. Therefore, we conclude the proof.
\end{proof}
\section{Proof of Theorem \ref{thm:main-thm}}\label{main-proof}\label{appendix-proof}
We first show that $ W^t \in \cC^t $ with high probability if $ R^t $ is properly chosen, which ensures that Algorithm \ref{alg:OE} induces sufficient optimism. The following lemma is frequently used to provide a sufficient trustworthy radius for a confidence set and is first proved by \cite{kakade2020information}. We provide its proof for completeness.
\begin{lemma}[Confidence Ball]\label{lem:ball-radius}
    For all $t\in[T]$, we set $\cE_{cb}^t$ as the event that $W^*$ falls in $\cC^t$, i.e., 
    $$
    \cE_{cb}^t =\bigg\{\big\|\left(\bar{W}^{t}-W^{\star}\right)\left(\Lambda^{t}\right)^{1 / 2}\big\|_{2}^{2}\leq R^{t}\bigg\},
    $$
    and $\cE_{cb}$ as the event that all $W^t$ falls in $\cC^t$, and $\cE_{cb} = \cap_{t=1}^T \cE_{cb}^t$.
	Let $$
	R^{t}=2 \lambda\|W^{\star}\|_{2}^{2}+8 \sigma^{2}\big(d \log (5)+2 \log (t)+\log (4)+\log (\operatorname{det}(\Lambda^{t}) / \operatorname{det}(\Lambda^{0})/\delta)\big),
	$$
	We have
	$$
	\sum_{t=0}^{\infty}\P(\bar{\cE}_{cb}^t)= \sum_{t=0}^{\infty} \PP\left(\left\|\left(\bar{W}^{t}-W^{\star}\right)\left(\Lambda^{t}\right)^{1 / 2}\right\|_{2}^{2}>R^{t}\right) \leq \delta/2.
	$$
\end{lemma}
\begin{proof}
The center of the confidence ball, $\bar{W}^{t}$, is the minimizer of the ridge regression objective, and its closed-form expression is
$$
\bar{W}^{t}=\sum_{\tau=1}^{t} \sum_{h=1}^{H} s_{h+1}^{\tau}\phi(s_h^{\tau},a_h^\tau)^{\top}(\Lambda^{t})^{-1},
$$
where $\Lambda^{t}=\lambda I+\sum_{\tau=1}^{t} \sum_{h=1}^{H} \phi(s_h^{\tau},a_h^\tau)^{\top}\phi(s_h^{\tau},a_h^\tau)^{\top}$. Since $s_{h+1}^{\tau}=W^{\star} \phi(s_h^{\tau},a_h^\tau)+\epsilon_{h}^{\tau}$ with $\epsilon_{h}^{\tau} \sim \mathcal{N}\left(0, \sigma^{2} \cI\right)$, we have
$$
\begin{aligned}
\bar{W}^{t}-W^{\star}&=\sum_{\tau=1}^{t} \sum_{h=1}^{H} s_{h+1}^{\tau}\phi(s_h^{\tau},a_h^\tau)^{\top}(\Lambda^{t})^{-1}-W^{\star} \\
&=\sum_{\tau=1}^{t} \sum_{h=1}^{H}(W^{\star} \phi(s_h^{\tau},a_h^\tau)+\epsilon_{h}^{\tau})\phi(s_h^{\tau},a_h^\tau)^{\top}\left(\Lambda^{t}\right)^{-1}-W^{\star} \\
&=W^{\star}\bigg(\sum_{\tau=1}^{t} \sum_{h=1}^{H} \phi(s_h^{\tau},a_h^\tau)\phi(s_h^{\tau},a_h^\tau)^{\top}\bigg)(\Lambda^{t})^{-1}-W^{\star}+\sum_{\tau=1}^{t} \sum_{h=1}^{H} \epsilon_{h}^{\tau}\phi(s_h^{\tau},a_h^\tau)^{\top}(\Lambda^{t})^{-1} \\
&=-\lambda W^{\star}(\Lambda^{t})^{-1}+\sum_{\tau=1}^{t} \sum_{h=1}^{H} \epsilon_{h}^{\tau}\phi(s_h^{\tau},a_h^\tau)^{\top}(\Lambda^{t})^{-1}.
\end{aligned}
$$
For any $0<\delta_{t}<1$, using Lemma \ref{lem:self-normalized}, it holds with probability at least $1-\delta_{t}$,
$$
\begin{aligned}
\bigg\|(\bar{W}^{t}-W^{\star})(\Lambda^{t})^{1 / 2}\bigg\|_{2} &\leq\bigg\|\lambda W^{\star}(\Lambda^{t})^{-1 / 2}\bigg\|_{2}+\bigg\|\sum_{\tau=1}^{t} \sum_{h=1}^{H} \epsilon_{h}^{\tau}\phi(s_h^{\tau},a_h^\tau)^{\top}(\Lambda^{t})^{-1 / 2}\bigg\|_{2} \\
&\leq \sqrt{\lambda}\|W^{\star}\|_{2}+\sigma \sqrt{8 d \log (5)+8 \log \big(\operatorname{det}(\Lambda^{t}) \operatorname{det}(\Lambda^{0})^{-1} / \delta_{t}\big)},
\end{aligned}
$$
where the first inequality follows from the triangle inequality. Therefore, we obtain $\PP\big(\overline{\mathcal{E}}_{cb}^t\big) \leq \delta_{t}$.
We seek to bound $\sum_{t=0}^{\infty} \PP(\bar{\mathcal{E}}_{c b}^t) .$ Note that at $t=0$ we have initialized $\cC^0$ to contain $W^{\star}$, we have $\PP(\bar{\mathcal{E}}_{c b}^0)=0$. For $t \geq 1$, let us assign failure probability $\delta_{t}=(3\delta / \pi^{2}) / t^{2}$ for the $t$-th event.We obtain
$$
\sum_{t=1}^{\infty} \PP(\bar{\mathcal{E}}_{c b}^t)\leq\sum_{t=1}^{\infty}(\delta / t^{2})(3 / \pi^{2}) =\delta / 2 .
$$ 
Therefore, we conclude the proof of Lemma \ref{lem:ball-radius}.
\end{proof}
    For all $t\in[T]$, we set $\cE_{cb}^t$ as the event that $W^*$ falls in $\cC^t$, i.e., 
    $$
    \cE_{cb}^t =\bigg\{\big\|\left(\bar{W}^{t}-W^{\star}\right)\left(\Lambda^{t}\right)^{1 / 2}\big\|_{2}^{2}\leq R^{t}\bigg\},
    $$
    and $\cE_{cb}$ as the event that all $W^t$ falls in $\cC^t$, i.e., $\cE_{cb} = \cap_{t=1}^T \cE_{cb}^t$.
    We prove in Lemma \ref{lem:ball-radius} that $\sum_{t=1}^{\infty}\PP(\bar{\cE}_{cb}^{t})\leq \delta/2$, where $\bar{\cE}_{cb}^{t}$ denotes the complement of $\cE_{cb}^t$. The following lemma shows that by efficiently implementing the principle of optimism in Algorithm \ref{alg:OE}, the summation of the expected discrepancy of two projected kernel features is bounded. The main idea is to cast the projected kernel embedding to an initial state value function. Then by the value iteration implemented in Algorithm \ref{alg:OE}, we give a general bound for regret and violation in $\cO(\sqrt{T})$ scales.
	\begin{lemma}[Optimism for KNR] \label{lem:value-difference-knr}
	Suppose that Assumption  \ref{ass:bound-feature} holds. For Algorithms \ref{alg:OPDO} and \ref{alg:OE}, the following inequality holds with probability at least $1-\delta$,
	\#
	\EE\bigg[\sum_{t=1}^{T} {{\theta}}^t\cdot({\Psi}^t-{\Psi}^*)\bigg]\leq(1+\Gamma)\sqrt{H}
	.
	\#
\end{lemma}
\begin{proof}
    If $W^t$ falls in $\cC^t$ for all $t$, it holds that $\sum_{t=1}^{T} {{\theta}}^t\cdot({\Psi}^t-{\Psi}^*)\leq 0$  by optimism induced by line \ref{line:optimistic-update} in Algorithm 
    \ref{alg:OE}. We condition on the event $\cE_{cb}^t$ and the proof is done.
\end{proof}
The final step is to bound the estimation error of the visitation $ f({\Psi}^{\widehat{\pi}})-f(\widehat{{\Psi}}) $ and $ g({\Psi}^{\widehat{\pi}})-g(\widehat{{\Psi}}) $. With $ f $ and $ g $ being 1-Lipschitz, it suffices to bound $ T\|{\Psi}^{\widehat{\pi}}- \widehat{{\Psi}}\|_2 $. 
\begin{lemma}[Bound for estimation error]\label{lem:estimate-error}
	Suppose that Assumptions \ref{ass:slater}-\ref{ass:bound-feature} and \ref{ass:black-box} hold. For Algorithms \ref{alg:OPDO} and \ref{alg:OE}, we have \#
     T\|{\Psi}^{\widehat{\pi}}- \widehat{{\Psi}}\|_2 \leq CBHd\sqrt{T}\log\bigg(\frac{HT}{d\delta}\bigg)
	\# 
	holds with probability at least $1-\delta$, here  $C$ is an absolute constant only depends on $\sigma$.
\end{lemma}
\begin{proof}
	For all $ {x} = (x_h)_{h\in[H]} \in \R^{dH}$ with $ \|{x}\|_2 \leq1 $, we can consider $ \sum_{t=1}^{T}({\Psi}^{\widehat{\pi}}- \widehat{{\Psi}})\cdot{x} $ as the value difference of two processes, with cost at stage $h$ defined as $c_h(s_h,a_h) = \psi(s_h,a_h)\cdot x_h$. For simplicity, we denote $x\cdot({\Psi}^{\pi_t}-{\Psi}^t)= V_1^{\pi_t}- V_1^t $. In the following analysis we condition on the event $\cE_{cb}$,
further estimate the value difference. With Lemma \ref{lem:simulation-lemma1}, we have
\#\label{eq:value-norm-bound}
	 \sum_{t=1}^{T}\big(V_1^{\pi_t}(s)-V_1^t(s))
	 &\lesssim \sum_{t=1}^{T}B\sqrt{H}\mathbb{E}_{\pi_t}\biggl[\sum_{h=1}^{H} \big\|\big(W^{\star}-W^t\big) \phi(s_{h}^t, a_{h}^t)\big\|_{2}\big| \cH_t\biggr],
\#
Here $\{\cH_t\}_{t\in[T]}$ is the history before episode $t$, and the inequality holds by Lemma \ref{lem:simulation-lemma1}. For $W^{*} \in$ $\cC^t$, we have
\#\label{eq:bound-trans-norm}
\big\|\big(\widehat{W}^{t}-W^{*}\big) \phi(s_{h}^t, a_{h}^t)\big\|_{2} &\leq\big\|\big(\widehat{W}^{t}-W^{*}\big)(\Lambda^{t})^{1 / 2}\big\|_{2}\big\|(\Lambda^{t})^{-1/2} \phi(s_{h}^t, a_{h}^t)\big\|_{2}\notag \\
&\leq\bigg(\big\|\big(\widehat{W}^{t}-\bar{W}^{t}\big)\big(\Lambda^{t}\big)^{1 / 2}\big\|_{2}+\big\|\big(\bar{W}^{t}-W^{*}\big)(\Lambda^{t})^{1 / 2}\big\|_{2}\bigg)\big\|\phi(s_{h}^t, a_{h}^t)\big\|_{(\Lambda^{t})^{-1}}\notag\\
&\leq 2 \sqrt{R^{t}}\big\|\phi(s_{h}^t, a_{h}^t)\big\|_{(\Lambda^{t})^{-1}} .
\#
Summing up \eqref{eq:bound-trans-norm} over $h \in [H]$, we obtain
\begin{equation}\label{eq:rand-2}
\begin{aligned}
\sum_{h=1}^H\big\|\big(W^{\star}-W^t\big) \phi(s_{h}^t, a_{h}^t)\big\|_{2}
&\leq 2\sqrt{R^t}\sum_{h=1}^{H}\|\phi(s_{h}^t, a_{h}^t)\|_{(\Lambda^{t})^{-1}}.
\end{aligned}
\end{equation}
Plugging \eqref{eq:rand-2} back to \eqref{eq:value-norm-bound}, we have the following holds with probability at least $1-\delta$,
\begin{equation*}
     \begin{aligned}
\sum_{t=1}^{T}\big(V_1^{\pi_t}(s)-V_1^t(s)\big)
	 &\lesssim B\sqrt{H}\sum_{t=1}^T\EE\bigg[\sqrt{R^T}\sum_{h=1}^H\|\phi(s_h^t,a_h^t)\|_{(\Lambda^t)^{-1}}\,\bigg|\, \cH_t\bigg]\\
	 &\lesssim B\sqrt{H\sigma^2 d\log\bigg(\frac{HT}{d\delta}\bigg)}\sum_{t=1}^{T} \EE\bigg[\sum_{h=1}^H\|\phi(s_h^t,a_h^t)\|_{(\Lambda^t)^{-1}}\bigg|\, \cH_t\bigg]\\
	 &\lesssim B\sqrt{H\sigma^2 d\log\bigg(\frac{HT}{d\delta}\bigg)} \sum_{t=1}^T \sum_{h=1}^H\|\phi(s_h^t,a_h^t)\|_{(\Lambda^t)^{-1}} + \sum_{t=1}^T M_{t}^T,
	 \\
% 	 &\leq BH\sqrt{T\sigma^2 d\log(\frac{HTd}{\delta})}\cdot\EE\bigg[\sum_{t=1}^T\sum_{h=1}^H\|\phi(s_h^t,a_h^t)\|_{(\Lambda^t)^{-1}}^2\bigg|\, \cH_t\bigg]^{1/2}\\
	 &\leq CB\sqrt{TH d\log\bigg(\frac{HT}{d\delta}\bigg)}\bigg(\sum_{t=1}^T\sum_{h=1}^H\|\phi(s_h^t,a_h^t)\|_{(\Lambda^t)^{-1}}^2\bigg)^{1/2}+ H\sqrt{2T\log(4/\delta)},
	 \end{aligned}
\end{equation*}
holds with probability at least $1-\delta$. Here $C$ is some absolute constant that only depends on $\sigma$. The second inequality comes from the fact that $R^t$ is non-decreasing, and 
$$
\begin{aligned}
R^{T}&=\max\{2\sigma^{2},2\}+8 \sigma^{2}\bigg(d\log (5)+2 \log (T)+\log (4)+\log \big(\operatorname{det}(\Lambda^{T}) \operatorname{det}(\Lambda^{0})^{-1}\big)/\delta\bigg) \\
&\leq C' \sigma^{2}\bigg(d+\log (T)+\log \big(\operatorname{det}(\Lambda^{T}) \operatorname{det}\big(\Lambda^{0}\big)^{-1}/\delta\big)\bigg)\lesssim \sigma^2(d\log(TH/d\delta))
\end{aligned}
$$
and $\lambda=\max\{\sigma^{2} ,1\}$. The third inequality decompose the expectation term into a elliptical potential summation and a martingale difference series. The last inequality comes from the martingale difference is bounded by $H$, and with Hoeffding's inequality, we have $$
\PP(|\sum_{t=1}^T M_t^T| \geq s) \leq 2 \exp(\frac{-s^2}{2TH^2}),
$$
set $s = H\sqrt{2T\log(4/p)}$ and we prove that $\sum_{t=1}^T M_t^T \leq  H\sqrt{T\log(4/p)}$ with probability at least $1-\delta/2$. Since we condition on $\cE_{cb}$, which holds with probability at least $1-\delta/2$, the inequality holds with probability at least $1-\delta$. Next, we bound the elliptical potential term. By Lemma \ref{lem:potential-ellip}, we have $$
\sum_{t=1}^T\sum_{h=1}^H\|\phi(s_h^t,a_h^t)\|_{(\Lambda^t)^{-1}}^2 \leq 2H \log \big(\operatorname{det}(\Lambda^{T}) \operatorname{det}(\Lambda^{0})^{-1}\big) \lesssim dH\log(\frac{TH}{d}),
$$
where the third inequality comes from \ref{lem:log-sum}. Combining the results above we have $$ \sum_{t=1}^{T}\big(V_1^{\pi_t}(s)-V_1^t(s)\big) \lesssim CBHd\sqrt{T}\log\bigg(\frac{HT}{d\delta}\bigg),
$$
here $C$ is an absolute constant that only relates to $\sigma$. Since the argument above holds for all $ x $ with $ \|x\|_2 \leq1 $, set $x = (\Psi^{\pi_t} - \Psi^t)/\|\Psi^{\pi_t} - \Psi^t\|_2$, and we conclude the proof of Lemma \ref{lem:estimate-error}.
\end{proof}

\section{Proof for Theorem \ref{thm:main-thm-lowrank}}\label{appendix:proof-lowrank}
In this section we give a detailed proof for Theorem \ref{thm:main-thm-lowrank}. The main tool is the MLE fundamental theorem and Bernstein's inequality for martingales.
\begin{proof}
First, we prove that the choice of the confidence set $\cC^t$ is fully efficient, i.e. $\cP^* \in \cC^t$ with high probability.

\begin{lemma}
With probability at least $1-\delta$, we have the true underlying transition kernel $\cP^*_h: \cS\times\cS\times\cA\rightarrow\RR$ lies in the confidence set $\cC^t$ for all $t\in[T]$ and $h\in[H]$, i.e., $$
\EE_{\cD_h^t}[\|\cP_h^*(\cdot|s_h,a_h)-\hat{\cP}_h^t(\cdot|s_h,a_h)\|_1^2] \leq c\frac{\log(TH|\Theta||\Upsilon|/\delta)}{t},
$$
where $\EE_\cD[f(s,a)]$ takes the average of $f$ on the dataset $\cD$.
\end{lemma}
\begin{proof}
By the construction of $\cD^t_h$ in Algorithm \ref{alg:OE_low_rank}, we have $s_h^t \sim \pi_t$ and $a_h^t\sim U(\cA) $. Recall that $$\hat{\rho}^t_h(s_h,a_h) = \frac{1}{t-1}\sum_{i\in[t-1]} d_{\pi^i,h}(s_h)u(a_h),$$
Therefore, $\EE_{\cD_h^t}[\|\cP_h^*(\cdot|s_h,a_h)-\hat{\cP}_h^t(\cdot|s_h,a_h)\|_1^2]$ is a empirical realization of the visitation measure for $(s_h,a_h)$ under the Markov policy $\hat{\rho}^t$. For notation simplicity, for each $(h,t) \in [H]\times[T]$, define $\cF_{t,h}$ to be to the $\sigma$-algebra generated by the trajectories,
$$
\cF_{t,h} = \sigma(\{(s_i^\tau,a_i^\tau)\}_{(i,\tau)\in[H]\times[t-1]}\cup \{(s_i^t,a_i^t)\}_{i\in[h-1]}) ,
$$
since $\pi^\tau = \pi^i(s_1^1,a_1^1,...,s_H^{\tau-1})$ and is measurable with respect to $\cF_{t,h}$, we have $$
\begin{aligned}
&t(\EE_{\cD_h^t}[\|\cP_h^*(\cdot|s_h,a_h)-\hat{\cP}_h^t(\cdot|s_h,a_h)\|_1^2] - \EE_{\hat{\rho}_h^t}[\|\cP_h^*(\cdot|s_h,a_h)-\hat{\cP}_h^t(\cdot|s_h,a_h)\|_1^2])\\
&\qquad= \sum_{\tau \in [t]} \|\cP_h^*(\cdot|s_h^\tau,a_h^\tau)-\hat{\cP}_h^t(\cdot|s_h^\tau,a_h^\tau)\|_1^2 -\EE_{s_h\sim \pi^\tau, a_h \sim U(\cA)} \|\cP_h^*(\cdot|s_h,a_h)-\hat{\cP}_h^t(\cdot|s_h,a_h)\|_1^2
\end{aligned}
$$
being a martingale process with respect to the filtration $\{\cF_{t,h}\}_{(h,t)\in[H]\times[T]}$ for all $(h,t) \in [H] \times [T]$. Therefore, by applying Freedman inequality (Lemma \ref{lem:bernstein}), \ with probability at least $1-\delta$ that $$
\begin{aligned}
&\big|\EE_{\cD_h^t}[\|\cP_h^*(\cdot|s_h,a_h)-\hat{\cP}_h^t(\cdot|s_h,a_h)\|_1^2] - \EE_{\hat{\rho}_h^t}[\|\cP_h^*(\cdot|s_h,a_h)-\hat{\cP}_h^t(\cdot|s_h,a_h)\|_1^2]\big| \\
&\qquad \leq \sqrt{\frac{2\operatorname{Var}_{\hat{\rho}_h^t}\big[\|\cP_h^*(\cdot|s_h,a_h)-\hat{\cP}_h^t(\cdot|s_h,a_h)\|_1^2\big]\log(2TH/\delta)}{t}}+ \frac{\log(2TH/\delta)}{3t}\\
&\qquad \leq \sqrt{\frac{8\EE_{\hat{\rho}_h^t}\big[\|\cP_h^*(\cdot|s_h,a_h)-\hat{\cP}_h^t(\cdot|s_h,a_h)\|_1^2\big]\log(2TH/\delta)}{t}}+ \frac{\log(2TH/\delta)}{3t},
\end{aligned}
$$
where the second inequality follows from $ \|\cP_h^*(\cdot|s_h,a_h)-\hat{\cP}_h^t(\cdot|s_h,a_h)\|_1^2 \leq 4$. Recall that with Lemma \ref{lem:mle-guarantee}, we have $$\label{eq:mle}
\EE_{\hat{\rho}^t_h}\big[\|\cP_h^*(\cdot|s_h,a_h)-\hat{\cP}_h^t(\cdot|s_h,a_h)\|_1^2 \leq c\frac{\log(TH|\Theta||\Upsilon|/\delta)}{t}.
$$
Therefore, we have
\begin{align*}\big|\EE_{\cD_h^t}[\|\cP_h^*(\cdot|s_h,a_h)-\hat{\cP}_h^t(\cdot|s_h,a_h)\|_1^2] - \EE_{\hat{\rho}_h^t}[\|\cP_h^*(\cdot|s_h,a_h)-\hat{\cP}_h^t(\cdot|s_h,a_h)\|_1^2]\big| \leq c\frac{\log(TH|\Theta||\Upsilon|/\delta)}{t}.
\end{align*}
Summing this and \ref{eq:mle} and we conclude the proof.

\end{proof}
The next lemma ensures that as we explore and shrink the radius of the confidence set, the statistical distances between the MLE estimation and all transitions in the confidence set uniformly decrease in a $\tilde{O}(1/t)$ manner.
\begin{lemma}
For all $\cP = \{\cP_h\}_{h\in[H]} \in \cC^t$ and all $t\in[T]$, with probability at least $1-\delta$, we have $$
\EE_{\hat{\rho}^t}[\|\cP_h(\cdot|s_h,a_h)-\cP_h^*(\cdot|s_h,a_h)\|_1^2] \leq c^{\prime}\frac{\log(TH|\Theta||\Upsilon|/\delta)}{t}.
$$
\end{lemma}
\begin{proof}
By the construction of $\cC^t$, we have the following inequality holds for all $\cP = \{\cP_h\}_{h\in[H]}$ in $\cC^t$ with high probability,
$$
\begin{aligned}
&\EE_{\cD^t_h}[\|\cP_h(\cdot|s_h,a_h)-\cP_h^*(\cdot|s_h,a_h)\|_1^2] \\
&\qquad\leq 2(\EE_{\cD^t_h}[\|\cP_h(\cdot|s_h,a_h)-\cP_h^t(\cdot|s_h,a_h)\|_1^2]+\EE_{\cD^t_h}[\|\cP_h^t(\cdot|s_h,a_h)-\cP_h^*(\cdot|s_h,a_h)\|_1^2])
\\
&\qquad \leq 2c\frac{\log(TH|\Theta||\Upsilon|/\delta)}{t},
\end{aligned}
$$
The first inequality comes from $(a+b)^2\leq 2(a^2+b^2)$. Define
$$
A(\cP_h) = \EE_{\hat{\rho}_h^t}[\|\cP_h(\cdot|s_h,a_h)-\cP_h^*(\cdot|s_h,a_h)\|_1^2] - \EE_{\cD_h^t}[\|\cP_h(\cdot|s_h,a_h)-\cP_h^*(\cdot|s_h,a_h)\|_1^2].
$$
We have 
$$
\EE_{\hat{\rho}_h^t}[\|\cP_h(\cdot|s_h,a_h)-\cP_h^*(\cdot|s_h,a_h)\|_1^2] \leq A(\cP_h) + 2c\frac{\log(TH|\Theta||\Upsilon|/\delta)}{t},
$$
for all $(h,t) \in [H]\times[T]$. Applying Freedman's inequality again, for any $\{\cP_h\}_{h\in[H]} \in \cC^t$ and $t\in[T]$, we have with probability $1-\delta$ that
$$\label{eq:boundforall}
\begin{aligned}
A(\cP_h) & \leq \sqrt{\frac{c_1\operatorname{Var}_{\hat{\rho}^t}\big[\|\cP_h^*(\cdot|s_h,a_h)-\hat{\cP}_h^t(\cdot|s_h,a_h)\|_1^2\big]\log(TH|\Theta||\Upsilon|/\delta)}{t}}+ \frac{c_2\log(TH|\Theta||\Upsilon| / \delta)}{t}\\
&\leq \sqrt{\frac{c_{1} \mathrm{E}_{\hat{\rho}^t_h}\left[\|\cP_h(\cdot|s_h,a_h)-\cP_h^*(\cdot|s_h,a_h)\|_1^4\right] \log (TH|\Theta||\Upsilon| / \delta)}{t}}+\frac{c_{2} \log (TH|\Theta||\Upsilon| / \delta)}{t} \\
&\leq \sqrt{\frac{4 c_{1} \mathrm{E}_{\hat{\rho}^t_h}\left[\|\cP_h(\cdot|s_h,a_h)-\cP_h^*(\cdot|s_h,a_h)\|_1^2\right] \log (H|\Theta||\Upsilon| / \delta)}{t}}+\frac{c_{2} \log (H|\Theta||\Upsilon| / \delta)}{t},
\end{aligned}
$$
where the first inequality comes from Bernstein's inequality, and the second inequality comes from the fact that $\|\cP(\cdot|s,a) - \cP^\prime(\cdot|s,a)\|_1^2\leq4$ for two probability distributions. Denote $\xi = \log (|\Theta||\Upsilon|TH / \delta)/t$ and taking square in both side of the \eqref{eq:boundforall}, we have
$$
\begin{aligned}
A^{2}(\cP_h) & \lesssim\left(\sqrt{\frac{c(A(\cP_h)+\xi) \log (TH|\Theta||\Upsilon| / \delta)}{t}}+\frac{c \log (|\Theta||\Upsilon| / \delta)}{t}\right)^{2} \\
&\lesssim \frac{(A(\cP_h)+\xi) \log (TH|\Theta||\Upsilon| / \delta)}{t}+\left\{\frac{c \ln (TH|\Theta||\Upsilon| / \delta)}{t}\right\}^{2}\\
&\lesssim \frac{(A(\cP_h)+\xi) \log (TH|\Theta||\Upsilon| / \delta)}{t}+\frac{c_{2} \log (TH|\Theta||\Upsilon| / \delta)}{t} . \\
&\lesssim \frac{(A(\cP_h)+1 / t \log (TH|\Theta||\Upsilon| / \delta)) \log (TH|\Theta||\Upsilon| / \delta)}{t} .
\end{aligned} 
$$
Then, we have
$$
A^{2}(\cP_h)-B_{1} A(\cP_h)-B_{2} \leq 0, \quad B_{1}=c \log (TH|\Theta||\Upsilon| / \delta) / t, \quad B_{2}=c(1 / t)^{2} \log (TH|\Theta||\Upsilon| / \delta)^{2}
$$
holds for all $\{\cP_h\}_{h\in[H]} \in \cC^t$ for all $t\in[T]$. This concludes
$$
0 \leq A(\cP_h) \leq \frac{B_{1}+\sqrt{B_{1}^{2}+4 B_{2}}}{2} \leq c\left(B_{1}+\sqrt{B_{2}}\right) \leq c \frac{\log (TH|\Theta||\Upsilon|/ \delta)}{t} \lesssim \xi.
$$
Thus, by using the above $A(\cP) \lesssim \xi\left(\cP \in \cC^t\right)$ and , with probability $1-\delta$, we have
$$
\mathbb{E}_{\hat{\rho}^t_h}\left[\|\cP_h(\cdot|s_h,a_h)-\cP_h^*(\cdot|s_h,a_h)\|_1^2\right] \leq A(P)+c \xi \lesssim \xi, \quad \{\cP_h\}_{h\in[H]} \in \cC^t .
$$
\end{proof}
To conclude our proof, the final step is to bound $T\|\hat{\Psi}-\Psi^{\hat{\pi}} \|_2 = \|\sum_{t=1}^T(\Psi^t-\Psi^{\pi^t})\|_2$. To this end, we still consider to find an uniform upper bound for $\sum_{t=1}^{T}(\Psi^t-\Psi^{\pi^t})\cdot\theta$ , $\theta = (\theta_h)_{h\in [H]}$ with $\|\theta\|_2\leq1$. As in the case of KNR, we define $$
\begin{aligned}
&c_h^t(s,a) = \psi_h(s,a)\cdot\theta_h,\\
&\Psi^t\cdot{\theta } = \EE_{\pi^t,\cP^t}\bigg[\sum_{h=1}^H c_h^t(s_h,a_h)\bigg] = V_1^t,\\
&\Psi^{\pi^t}\cdot\theta= \EE_{\pi^t,\cP^*}\bigg[\sum_{h=1}^H r_h^t(s_h,a_h)\bigg] = V_1^{\pi^t},
\end{aligned}
$$
With standard notations in reinforcement learning, we can define value function $V_h(s,a)$ for all stage $h\in [H]$. Using the value-decomposition lemma, we decompose the value difference  $V_1^t-V_1^{\pi^t}$,
\begin{equation}\label{eq:bound-difference}
\begin{aligned}
\sum_{t=1}^{T}(\Psi^t-\Psi^{\pi^t})\cdot\theta &=\sum_{t=1}^T V_1^t-V_1^{\pi^t}\\
&= \sum_{h=1}^H\sum_{t=1}^T \EE_{\pi^t}\big[\PP_h^*V_{h+1}^{\pi^t}(s_h,a_h) - \PP^t_h V_{h+1}^{\pi^t}(s_h,a_h)\big]\\
&\leq \sum_{h=1}^H\sum_{t=1}^T B\sqrt{H} \EE_{\pi^t}\big[\|\cP^t_h(\cdot|s_h,a_h)-\cP^*_h(\cdot|s_h,a_h)\|_1\big],
\end{aligned}
\end{equation}
here the second equation comes from the value difference lemma, and the third inequality comes from the fact that $\|V_1^\pi\|_{\infty} \leq B\sqrt{H}$, since $\|\theta\|_2\leq1$ and $\|\psi_h(s,a)\|_2\leq B$. The next lemma shows that we can upper bound $\EE_{\pi^t}[H(s_h,a_h)]$ using $\EE_{\pi^t}\|\phi_h^*(s_h,a_h)\|_{{\Sigma}^{-1}_{\rho^t,\phi^*_{h-1}}}$ once we can upper bound $\EE_{\hat{\rho}^t}[H^2(s_h,a_h)]$.
\begin{lemma}[One step back inequality]\label{lem:look_back}
Take any $H \in \mathcal{S} \times \mathcal{A} \rightarrow \mathbb{R}$ such that $\|H\|_{\infty} \leq B$. Then,
$$
\EE_{\pi}[H(s_h,a_h)] \leq \EE_\pi\|\phi_{h-1}^{\star}(s_{h-1},a_{h-1})\|_{\Sigma_{\rho^t_{h-1}, \phi_{h-1}^\star}^{-1}}\sqrt{t|\mathcal{A}| \EE_{\hat{\rho}^t_h}[H^{2}(s, a)]+\lambda d B^{2}},
$$
where $\Sigma_{\rho^t_h, \phi^{\star}_h}=t \mathbb{E}_{\rho_h^t}\left[\phi_h^{\star}(s_h, a_h) \phi_h^{\star}(s_h, a_h)^{\top}\right]+\lambda I$. Note that here the parameter $\lambda$ and the matrix $\Sigma_{\rho^t,\phi_h^*}$ doesn't occur in the actual implementation.
\end{lemma}
\begin{proof}
First, we have $$
\begin{aligned}
\EE_{\pi}[H(s_h,a_h)]  &=\EE_{(s_{h-1},a_{h-1})\sim\pi,s_h\sim \cP(s_{h-1},a_{h-1}),a_h\sim \pi}[H(s_h,a_h)] \\
&=\EE_\pi\bigg[\phi^*_{h-1}(s_{h-1},a_{h-1})^\top\int\sum_{a_h} \mu^*_{h-1}(s_h)\pi_h(a_h|s_h) H(s_h,a_h) d\nu\bigg]\\
&\leq \EE_\pi\bigg[\|\phi^*_{h-1}(s_{h-1},a_{h-1})\|_{\Sigma_{\rho^t_h, \phi_{h-1}^\star}^{-1}} \bigg\|\int\sum_{a_h} \mu^*_{h-1}(s_h)\pi_h(a_h|s_h) H(s_h,a_h) d\nu \bigg\|_{\Sigma_{\rho^t_h, \phi_{h-1}^\star}}\bigg],
\end{aligned}
$$
where the third inequality comes from Cauchy's inequality. Here, we have
$$
\begin{aligned}
&\left\|\int \sum_{a_h} \mu_{h-1}^{\star}(s) \pi_h(a_h |s_h) H(s_h, a_h) d\nu(s)\right\|_{\Sigma_{\rho^t_{h-1}, \phi_{h-1}^\star}}^{2} \\
&\leq\left\{\int \sum_{a_h} \mu_{h-1}^{\star}(s_h) \pi_h(a_h|s_h) H(s_h, a_h) d\nu(s)\right\}^{\top}\left\{t \mathbb{E}_{\rho^t_{h-1}}\left[\phi_{h-1}^{\star}(s_{h-1}, a_{h-1})\left\{\phi^{\star}_{h-1}(s_h, a_h)\right\}^{\top}\right]+\lambda I\right\}\left\{\int \sum_{a} \mu_{h-1}^{\star}(s_h) \pi_h(a_h|s_h) H(s_h, a_h) d\nu(s)\right\} \\
&\leq t \mathbb{E}_{\rho^t_{h-1}}\left\{\left[\int \sum_{a} \mu_{h-1}^{\star}(s_{h})^{\top} \phi_{h-1}^{\star}(s_{h-1},a_{h-1}) \pi_h(a_h \mid s_h) H(s_h, a_h) d\nu(s)\right]^{2}\right\}+\lambda d B^{2} \\
&\leq t\left\{\mathbb{E}_{(s_{h-1},a_{h-1}) \sim \rho^t_{h-1}, s_h \sim P^{\star}(s_{h-1},a_{h-1}), a_h \sim \pi(s)}\left[H^{2}(s_h,a_h)\right]\right\}+\lambda d B^{2},
\end{aligned}
$$
where the last inequality comes from Jensen's inequality. Further, we have that $$
\begin{aligned}
\mathbb{E}_{(s_{h-1},a_{h-1}) \sim \rho^t, s_h \sim P^{\star}(s_{h-1},a_{h-1}), a_h \sim \pi(s)}\left[H^{2}(s_h, a_h)\right]
&\leq 
|\cA|\mathbb{E}_{(s_{h-1},a_{h-1}) \sim \rho^t, s_h \sim P^{\star}(s_{h-1},a_{h-1}), a_h \sim U(\cA)}\left[H^{2}(s_h, a_h)\right]\\
&= |\cA|\EE_{\hat{\rho}_h^t}[H^2(s_h,a_h)]
\end{aligned}
$$
which concludes the proof.
\end{proof}
We then condition on the event
$$
\EE_{\hat{\rho}_h^t}[\|\cP^t_h(\cdot|s_h,a_h)-\cP^*_h(\cdot|s_h,a_h)\|_1^2]\leq c\frac{\log(TH|\Upsilon||\Theta|)}{t}, \forall (h,t) \in [H]\times[T],
$$ which holds with probability at least $1-\delta$, and use Lemma \ref{lem:look_back} on  \eqref{eq:bound-difference} by setting $\pi = \pi^{t}$ for $\|\cP^t_h(\cdot|s_h,a_h)-\cP^*_h(\cdot|s_h,a_h)\|_1$, we have
$$
\begin{aligned}
&\sum_{h=1}^H\sum_{t=1}^T \EE_{\pi^t}\big[\|\cP^t_h(\cdot|s_h,a_h)-\cP^*_h(\cdot|s_h,a_h)\|_1\big] \\
&\qquad\leq \sum_{h=1}^H\sum_{t=1}^T\EE_{\pi^t}\|\phi_h^*(s_h,a_h)\|_{\Sigma_{\rho_h^{t},\phi_h^*}^{-1}} \sqrt{t|\cA|\EE_{\hat{\rho}^{t}_h}[\|\cP_h^t-\cP_h^*\|_1^2]+4\lambda d}\\
&\qquad\lesssim \sum_{h=1}^H\sum_{t=1}^T \EE_{\pi^t}\|\phi_h^*(s_h,a_h)\|_{\Sigma_{\rho_h^{t},\phi_h^*}^{-1}} \sqrt{|\cA|\log(TH|\Theta||\Upsilon|/\delta)+\lambda d},
\end{aligned}
$$
here the first inequality comes from the one-step back inequality and the fact that every term in the summation is positive, the second inequality comes from our condition event. We also have$$
\sqrt{|\cA|\log(TH|\Theta||\Upsilon|/\delta)+\lambda d} \lesssim \sqrt{|\cA|\lambda d\log(TH|\Theta||\Upsilon|/\delta)} = \xi_T,
$$
therefore $$
\begin{aligned}
&\sum_{h=1}^H\sum_{t=1}^T \EE_{\pi^t}\big[\|\cP^t_h(\cdot|s_h,a_h)-\cP^*_h(\cdot|s_h,a_h)\|_1\big] \\
&\qquad \lesssim \sum_{h=1}^H \xi_T \big(\sum_{t=1}^T \EE_{\pi^t}\|\phi_h^*(s_h,a_h)\|_{\Sigma_{\rho_h^{t},\phi_h^*}^{-1}}\big)\\
&\qquad \leq \sum_{h=1}^H\xi_T\cdot\sqrt{T\sum_{t=1}^T\EE_{\pi^t}\bigg[\phi_h^*(s_h,a_h)^\top\Sigma_{\rho_h^{t},\phi_h^*}^{-1}\phi_h^*(s_h,a_h)\bigg]},
\end{aligned}
$$
where the second inequality comes from Jensen's inequality. By Lemma \ref{lem:matrix-sum} and Lemma \ref{lem:log-sum}, we have $$
\begin{aligned}
\sum_{t=1}^T\EE_{\pi^t}\bigg[\phi_h^*(s_h,a_h)^\top\Sigma_{\rho_h^{t},\phi_h^*}^{-1}\phi_h^*(s_h,a_h)\bigg] &= \sum_{t=1}^T \operatorname{Tr}\bigg(\Sigma_{\rho_h^{t},\phi_h^*}^{-1}\cdot\EE_{\pi_t}\big[\phi_h^*(s_h,a_h)\phi_h^*(s_h,a_h)^\top\big]\bigg)\\
&\leq 2\bigg(\log\operatorname{det}\big(\Sigma_{\rho_h^{T},\phi_h^*}\big) - 2\log\operatorname{det}\big(\lambda I\big)\bigg)\\
&\leq d\log\bigg(1+\frac{T}{d\lambda}\bigg)
\end{aligned}
$$
holds for all $h\in [H]$. By then we have $$
\sum_{h=1}^H\sum_{t=1}^T \EE_{\pi^t}\big[\|\cP^t_h(\cdot|s_h,a_h)-\cP^*_h(\cdot|s_h,a_h)\|_1\big] \leq \sqrt{T\log\bigg(1+\frac{T}{d\lambda}\bigg)}|\cA|^{1/2}d\lambda^{1/2}\log(TH|\Theta||\Upsilon|/\delta),
$$
combine with \eqref{eq:bound-difference} and set $\lambda =1$, $\theta = (\hat{\Psi} - \Psi^{\hat{\pi}})/\|\hat{\Psi} - \Psi^{\hat{\pi}}\|_2$, we conclude the proof of Lemma \ref{lem:appendix-error-lowrank}

Combine Lemma \ref{lem:spurious-bound} and \ref{lem:appendix-error-lowrank} we finish the proof of Theorem \ref{thm:main-thm-lowrank}.
\end{proof}
\section{Lemmas for Optimization}\label{sec:optimization}
\subsection{Online learning}
Online learning involves two players: the adversary and the player. The online learning protocol is shown in Algorithm \ref{alg:protocol-online-learning}.

\begin{algorithm}[H]
	\caption{Protocol of Online Learning}
	\label{alg:protocol-online-learning}
	\begin{algorithmic}[1]
		\FOR{$t=1,\dots,T$}
		\STATE The player chooses an action $x_t$.
		\STATE The adversary picks a function $f_t$.
		\STATE The player obtains reward $f_t(x_t)$.
		\STATE The player learns via $f_t$.
		\ENDFOR
	\end{algorithmic}
\end{algorithm}

\noindent Note that there is no assumption on how the adversary will pick the function $f_t$, and it may be adversarially chosen. The player aims to minimize the regret:
\begin{equation}\label{eq:regret}
\operatorname{Regret}=\max _{x} \sum_{t=1}^{T} f_{t}(x)-\sum_{t=1}^{T} f_{t}(x_{t}),
\end{equation}
which measures the quality of the player's strategy $x_1,\dots,x_T$ compared with the single best decision in hindsight.

\noindent \textbf{Projected Subgradient Method.}
The projected subgradient method is a particular case of mirror descent/ascent with Euclidean distance. Applying this method to online learning produces a regret bound of the order $O(\sqrt{T})$.

\noindent Suppose that the actions $x_t$ are required to be contained in some convex set $\+X$, i.e., $x_t\in\+X$. Denote a subgradient of $f_t$ at $x_t$ by $g_t\in \partial f_t(x_t)$, $G$ and $R$ are two constants such that $\max_{x,y\in\+X}\|x-y\|_2\le R$ and $\max_{t\in[T]}\|\partial f_t(x_t)\|_2\le G$. We set the step length $\eta_t$ at the $t$-th iteration to $R/G\sqrt{t}$ if we do not know the number of iterations $T$ in advance and to $R/G\sqrt{T}$ if we have the knowledge of $T$. The latter case will leads to an upper bound with a smaller constant multiplicative factor. With these notations, the update rule of projected subgradient method can be expressed as
$$
x_{t+1}\leftarrow\mathop{\arg\max}_{x\in\+X}
\bigl\{
f_t(x_t)+\langle\eta_t g_t,x-x_t\rangle
-\|x-x_t\|_2^2/2\bigr\}.
$$
\noindent We describe the complete method in Algorithm \ref{alg:proj-subgrad}.

\begin{algorithm}[H]
	\caption{projected subgradient method}
	\label{alg:proj-subgrad}
	\begin{algorithmic}[1]
		\STATE Arbitrarily initialize $x_1\in\+X$.
		\FOR{$t=1,\dots,T-1$}
		\STATE Update $x_{t+1}\leftarrow\mathop{\arg\max}_{x\in\+X}
		\big\{
		f_t(x_t)+\langle\eta_t g_t,x-x_t\rangle
		-\|x-x_t\|_2^2/2\big\}$
		\ENDFOR
	\end{algorithmic}
\end{algorithm}

\noindent By this method, the regret is guaranteed to increase sublinearly as stated in the following theorem.

\begin{theorem}\label{thm:proj-subgrad}
	Using projected subgradient method mentioned in Algorithm \ref{alg:proj-subgrad}, it holds that for all $x$ in the convex set $ \cX $ we have
	\begin{equation*}
	\sum_{t=1}^{T} f_{t}(x)-\sum_{t=1}^{T} f_{t}(x_{t}) \le C R G\sqrt{T},
	\end{equation*}
	where $C$ is an absolute constant.
\end{theorem}
\begin{proof}
	See \citet{zinkevich2003online} for a detailed proof. Note that the choice of $ x $ is irrelevant in the proof.
\end{proof}
\subsection{Constrained Optimization}
In this subsection we consider a general constrained optimization and discuss its properties. We consider \#\label{eq:dual-prob}
f_{\text {opt}}=\min _{{x} \in X}\{f({x}): {g}({x}) \leq 0, {A x}+{b}=0\},
\#
where  and $f, g: \R \rightarrow(-\infty, \infty)$ are convex real-valued functions, ${A} \in \mathbb{R}^{p \times n}, {b} \in \mathbb{R}^{p}$.
We define a value function associated with \eqref{eq:dual-prob},
$$
v({u}, {t})=\min _{{x} \in X}\{f({x}): {g}({x}) \leq {u}, {A} {x}+{b}={t}\}.
$$
Furthermore, we define the dual problem to \eqref{eq:dual-prob}. The dual function is
$$
q(\lambda, \gamma)=\min _{x \in X}\left\{L({x}, \lambda, \gamma)=f({x})+\lambda^{T} {g}({x})+\gamma^{T}({A x}+{b})\right\},
$$
where $\lambda \in \mathbb{R}_{+}^{m}, \gamma \in \mathbb{R}^{p}$. The corresponding dual problem is
$$\label{eq:dual-optimal}
q_{\mathrm{opt}}=\max _{\lambda \in \mathbb{R}_{+}^{m}, \gamma \in \mathbb{R}^{p}}\{q(\lambda, \gamma):(\lambda, \gamma) \in \operatorname{dom}(-q)\}.
$$
Where $\operatorname{dom}(-q)=\left\{(\lambda, \gamma) \in \mathbb{R}_{+}^{m}, \gamma \in \mathbb{R}^{p}: q(\lambda, \gamma)>-\infty\right\} .$ Furthermore, we denote an optimal solution of \eqref{eq:dual-optimal} by $\lambda^{*}, \gamma^{*}$.

We make the following assumption which will be verified to hold. The assumption implies strong duality, i.e., $q_{\mathrm{opt}}=f_{\mathrm{opt}}$.

\begin{assumption}
The optimal value of \eqref{eq:dual-prob} is finite and exists a Slater point $\overline{\mathrm{x}}$ such that $g(\overline{x}<0$ and exists a point $\widehat{x} \in \operatorname{ri}(X)$ satisfying $A \widehat{x}+b=0$, where $\operatorname{ri}(X)$ is the relative interior of $X $.
\end{assumption} 
The following theorem is proved in \cite{beck2017first}.
\begin{theorem}\label{thm:beck-theorem}
The dual variable $(\lambda^{*}, \gamma^{*})$ is an optimal solution of \eqref{eq:dual-prob} if and only if
$$
-(\lambda^{*}, \gamma^{*}) \in \partial v(0,0),
$$
where $\partial f(x)$ denotes the set of all sub-gradients of $f$ at $\mathbf{x}$.
\end{theorem}
\begin{proof}
    See Theorem 3.59, \cite{beck2017first}.
\end{proof}
Using this result we arrive at the following theorem, which is a variant of \cite{beck2017first} , Theorem 3.60.
\begin{theorem}\label{thm:vio-bound}
Let $\lambda^{*}$ be an optimal solution of the dual \eqref{eq:dual-prob} and assume that $2\|\lambda^{*}\|_{1} \leq \rho .$ Let $\widetilde{{x}}$ satisfy ${A} \widetilde{{x}}+{b}=0$ and
$$
f(\widetilde{{x}})-f_{\text {opt }}+\rho\|[g(\widetilde{{x}})]_{+}\|_{\infty} \leq \delta
$$
Then we have
$$
\|[g(\widetilde{{x}})]_{+}\|_{\infty} \leq \frac{\delta}{\rho}.
$$
\end{theorem}
\begin{proof}
Let
$$
v({u}, {t})=\min _{{x} \in X}\{f({x}): g({x}) \leq {u}, Ax+b=t\}.
$$
Since $(-\lambda^{*}, \gamma^{*})$ is an optimal solution of the dual problem it follows by Theorem \ref{thm:beck-theorem} that $(-\lambda^{*}, \gamma^{*}) \in$ $\partial v(0,0) .$ Therefore, for any $(u,0) \in \operatorname{dom}(v)$,
\#\label{eq:vio-bound-1}
v(u,0)-v(0,0) \geq\left\langle-\lambda^{*}, u\right\rangle.
\#
Set $u=\widetilde{u}=[g(\widetilde{x})]_{+}$. Since $\overline{u} \geq 0$, we have
$$
v(\widetilde{u}, 0) \leq v(0, 0)=f_{\text{opt}} \leq f(\widetilde{x}).
$$
Thus, \eqref{eq:vio-bound-1} implies that
\#\label{eq:vio-bound-2}
f(\widetilde{x})-f_{\mathrm{opt}} \geq\langle-\lambda^{*}, \widetilde{u}\rangle.
\#
Thus, we obtain
$$\label{vio-bound-2}
\begin{aligned}
(\rho-\|\lambda^{*}\|_{1})\|\widetilde{u}\|_{\infty} &=-\|\lambda^{*}\|_{1}\|\widetilde{u}\|_{\infty}+\rho\|\widetilde{u}\|_{\infty} \\
& \leq\langle-\lambda^{*}, \widetilde{u}\rangle+\rho\|\widetilde{u}\|_{\infty} \\
&=f(\widetilde{x})-f_{\mathrm{opt}}+\rho\|\overline{u}\|_{\infty} \leq \delta,
\end{aligned}
$$
where the last relation follows from \eqref{eq:vio-bound-2}. Rearranging the terms and using the assumption $2\|\lambda^{*}\|_{1} \leq \rho$, we obtain
$$
\|[g(\widetilde{x})]_{+}\|_{\infty}=\|\overline{u}\|_{\infty} \leq \frac{\delta}{\rho-\|\lambda^{*}\|_{1}} \leq \frac{2}{\rho} \delta.
$$
Therefore, we conclude the proof of Theorem \ref{thm:beck-theorem}.
\end{proof}

For the solution of the dual function, the following lemma is an adjustment of \cite{beck2017first}.
\begin{theorem}\label{thm:bound-up}
	Let $\overline{x} \in X$ be a point satisfying $g(\bar{x})<0$ and $A \overline{x}+b=0 .$ Then, for any $\lambda, \gamma \in$ $\{\lambda \in \mathbb{R}_{+}^{m}, \gamma \in \mathbb{R}_{+}^{p}: q(\lambda, \gamma) \geq M\}$, we have
	$$
	\|\lambda\|_{1} \leq \frac{f(\bar{x})-M}{\min _{j\in[m]}\{-g_{j}(\bar{x})\}}.
	$$
\end{theorem}
\begin{proof}
	Let
	$$
	S_{M}=\{\lambda \in \mathbb{R}_{+}^{m}, \gamma \in \mathbb{R}_{+}^{p}: q(\lambda, \gamma) \geq M\}.
	$$
	By the definition of $S_M$, for any $\lambda, \gamma \in S_{M}$ we have
	\begin{align*}
	M &\leq q(\lambda, \gamma) \\
	&=\min _{x \in X}\{f(x)+\lambda^{T} g(x)+\gamma^{T}(Ax+b)\} \\
	&\leq f(\overline{x})+\lambda^{T} g(\overline{x})+\gamma^{T}(A \overline{x}+b) \\
	&=f(\overline{x})+\sum_{j=1}^{m} \lambda_{j} g_{j}(\overline{x}).
	\end{align*}
	Therefore, we obtain
	$$
	-\sum_{j=1}^{m} \lambda_{j} g_{j}(\overline{x}) \leq f(\overline{x})-M,
	$$
	which implies that for any $(\lambda, \gamma) \in S_{M}$,
	$$
	\sum_{j=1}^{m} \lambda_{j}=\|\lambda\|_{1} \leq \frac{f(\overline{x})-M}{\min _{j\in[m]}\{-g_{j}(\overline{x})\}}.
	$$
	Therefore, we conclude the proof of Theorem \ref{thm:bound-up}.
\end{proof}
A simple corollary gives an estimation of the optimal dual solution $ \lambda^* $.
\begin{corollary}\label{coro:bound-slater}
 Let $\overline{x} \in X$ be a point satisfying  $g(\bar{x})<0$ and $A \overline{x}+b=0$, and $\lambda^{*}$ be an optimal dual solution. Then, it holds that
	$$
	\|\lambda^{*}\|_{1} \leq \frac{f(\bar{x})-M}{\min _{j\in[m]}\{-g_{j}(\bar{x})\}}.
	$$
\end{corollary}
\begin{proof}
	 Since $ \left(\lambda^{*}, \gamma^{*}\right) \in S_{f_{\text {opt }}} $  be an optimal solution of the dual problem equation \ref{eq:dual-prob}, we finish the proof by Theorem \ref{thm:bound-up}.
\end{proof}
\subsection{Proof of Proposition \ref{prop:convexity}}\label{appendix-cvx-proof}
\begin{proof}
   To prove the convexity of \eqref{eq:convex-prob}, it suffices to show that $\cV$ is convex. We allow some initial randomizing mechanisms such that the policy $ \{\pi_h\}_{h\in[H]} $ not only rely on $ h $, but also depends on a randomizing mechanism. We may have a set of policies $\mathcal{U}$ and a distribution $q \in \Delta(\mathcal{U})$. Then the mixed policy $\hat{\pi}$ of  $\mathcal{U}$, is defined such that we choose some policy $\pi \in \mathcal{U}$ using $q$ and then the agent proceeds executing with only that policy \citep{altman1999constrained}. We have the following equality,
   $$
    {\Psi}^{\hat{\pi}} = \EE_{q}[{\Psi}^\pi],
   $$
   where the expectation is taken with respect to the underlying distribution $q$ and all policy $\pi \in\mathcal{U}$. When $q$ is set as the uniform distribution on set $\{\pi_k\}_{k\in[K]}$, we have
   $$
     {\Psi}^{\hat{\pi}} = \frac{1}{K}\sum_{k=1}^K {\Psi}^{\pi_k}.
   $$
   Since ${\Psi}^{\hat{\pi}} \in \cV$ with our definition, $\cV$ is a convex set. The optima of \eqref{eq:convex-prob} over the mixed policy will remain the same, and $ \cV $ is proved to be a convex set. The feasible set for $ \eqref{eq:convex-prob} $ is thus convex and the problem is indeed a convex optimization.
\end{proof}

\section{Auxiliary Results}
The difference of value functions between two MDPs has the following general decomposition, which is rather useful in our analysis.
\begin{lemma}[Value Difference Lemma]\label{lem:value-difference} Consider two MDPs $\big(\mathcal{S}, \mathcal{A}, \{\cP^1_h\}_{h=1}^{H},\{r^1_h\}_{h=1}^{H} \big)$ and 
$\big(\mathcal{S}, \mathcal{A}, \{\cP_h^2\}_{h=1}^{H}, \{r^2_h\}_{h=1}^{H}\big)$ and a given policy $\pi = \{\pi_h\}_{h\in [H]}$. Their corresponding value functions in the $h$-th horizon are $V_h^\pi$ and $V_h^{\pi'}$ respectively. Then for all $h\in [H]$ the following relation holds,
\#
V_h^\pi(s) - V^{\pi'}_h(s) &= 
\EE_{\pi,\cP}[\sum_{i=h}^{H}(r_i(s_i,a_i)-r^{'}_i(s_i,a_i))\mid s_h =s] \\
& \qquad\quad+ \EE_{\pi,\cP^{'}}[\sum_{i=h}^{H}(\mathbb{P}_i V_{i+1}^\pi(s_i,a_i) - {\mathbb{P}}_i^{'}V_{i+1}^\pi(s_i,a_i))\mid s_h =s]
\#
\end{lemma}
\begin{proof}
See Lemma E.15 in \cite{dann2017unifying} for details.
\end{proof}
We introduce the following lemma, which gives a self-normalized bound for vector value martingales\citep{abbasi2011improved}.
\begin{lemma}[Self-Normalized Bound for Vector-Valued Martingales]\label{lem:self-normal-vector}
Let $\{\varepsilon_{i}\}_{i=1}^{\infty}$ be a real-valued stochastic process with corresponding filtration $\{\mathcal{F}_{i}\}_{i=1}^{\infty}$ such that $\varepsilon_{i}$ is $\mathcal{F}_{i}$ measurable, $\mathbb{E}[\varepsilon_{i} \mid \mathcal{F}_{i-1}]=0$, and $\varepsilon_{i}$ is conditionally $\sigma$-sub-Gaussian with $\sigma \in \mathbb{R}^{+}$. Let $\{X_{i}\}_{i=1}^{\infty}$ be a stochastic process with $X_{i} \in \mathcal{H}$ (some Hilbert space) and $X_{i}$ being $\mathcal{F}_{t}$-measurable. Assume that a linear operator $V: \mathcal{H} \rightarrow \mathcal{H}$ is positive definite, i.e., $x^{\top} V x>0$ for any $x \in \mathcal{H} .$ For any $t$, define the linear operator $V_{t}=V+\sum_{i=1}^{t} X_{i} X_{i}^{\top}$ (here $x x^{\top}$ denotes outer-product in $\mathcal{H}$ ). With probability at least $1-\delta$, we have for all $t \geq 1$
$$
\bigg\|\sum_{i=1}^{t} X_{i} \varepsilon_{i}\bigg\|_{V_{t}^{-1}}^{2} \leq 2 \sigma^{2} \log \bigg(\frac{\operatorname{det}(V_{t})^{1 / 2} \operatorname{det}(V)^{-1 / 2}}{\delta}\bigg).
$$
\end{lemma}
\begin{proof}
For a detailed proof, see \cite{abbasi2011improved}.
\end{proof}
Lemma \ref{lem:self-normalized} can be generalized to the case of matrix-valued martingales.
\begin{lemma}[Self-Normalized Bound for Matrix-Valued Martingales]\label{lem:self-normalized}
Let $\left\{\varepsilon_{i}\right\}_{i=1}^{\infty}$ be a d-dimensional vector-valued stochastic process with corresponding filtration $\left\{\mathcal{F}_{i}\right\}_{i=1}^{\infty}$ such that $\varepsilon_{i}$ is $\mathcal{F}_{i}$ measurable, $\mathbb{E}\left[\varepsilon_{i} \mid \mathcal{F}_{i-1}\right]=0$, and $\varepsilon_{i}$ is conditionally $\sigma$-sub-Gaussian with $\sigma \in \mathbb{R}^{d}$ Let $\left\{X_{i}\right\}_{i=1}^{\infty}$ be a stochastic process with $X_{i} \in \mathcal{H}$ (some Hilbert space) and $X_{i}$ being $\mathcal{F}_{t}$ measurable. Assume that a linear operator $V: \mathcal{H} \rightarrow \mathcal{H}$ is positive definite. For any $t$, define the linear operator $V_{t}=V+\sum_{i=1}^{t} X_{i} X_{i}^{\top}$ Then, with probability at least $1-\delta$, we have for all $t$, we have:
$$
\bigg\|\sum_{i=1}^{t} \epsilon_{i} X_{i}^{\top} V_{t}^{-1 / 2}\bigg\|_{2}^{2} \leq 8 \sigma^{2} d \log (5)+8 \sigma^{2} \log \bigg(\frac{\operatorname{det}(V_{t})^{1 / 2} \operatorname{det}(V)^{-1 / 2}}{\delta}\bigg).
$$

\end{lemma}
\begin{proof}
   Denote $S=\sum_{i=1}^{t} \epsilon_{i} X_{i}^{\top}$. Let us form an $\epsilon$-net, in $\ell_{2}$ distance, $\mathcal{C}$ over the unit ball $\{w:$ $\left.\|w\|_{2} \leq 1, w \in \mathbb{R}^{d}\right\}$. Via a standard covering argument, we can choose $\mathcal{C}$ such that $\log (|\mathcal{C}|) \leq d \log (1+2 / \epsilon)$.

Consider a fixed $w \in \mathcal{C}$ and $w^{\top} S=\sum_{i=1}^{t} w^{\top} \epsilon_{i} X_{i}^{T}$. Note that $w^{\top} \epsilon_{i}$ is a $\sigma$-sub Gaussian due to $\|w\|_{2} \leq 1 .$ Hence, Lemma \ref{lem:self-normalized} implies that with probability at least $1-\delta$, for all $t$
$$
\left\|V_{t}^{-1 / 2} \sum_{i=1}^{t} X_{i}\left(w^{\top} \epsilon_{i}\right)\right\|_{2} \leq \sqrt{2} \sigma \sqrt{\log \left(\frac{\operatorname{det}\left(V_{t}\right)^{1 / 2} \operatorname{det}(V)^{-1 / 2}}{\delta}\right)}.
$$
Now apply a union bound over $\mathcal{C}$, we get that with probability at least $1-\delta$,
$$
\forall w \in \mathcal{C}:\left\|V_{t}^{-1 / 2} \sum_{i=1}^{t} X_{i}\left(w^{\top} \epsilon_{i}\right)\right\|_{2} \leq \sqrt{2} \sigma \sqrt{d \log (1+2 / \epsilon)+\log \left(\frac{\operatorname{det}\left(V_{t}\right)^{1 / 2} \operatorname{det}(V)^{-1 / 2}}{\delta}\right)}.
$$
For any $w$ with $\|w\|_{2} \leq 1$, there exists a $w^{\prime} \in \mathcal{C}$ such that $\left\|w-w^{\prime}\right\|_{2} \leq \epsilon$. Hence, for all $w$ such that $\|w\|_{2} \leq 1$,
$$
\begin{aligned}
\left\|V_{t}^{-1 / 2} \sum_{i=1}^{t} X_{i}\left(w^{\top} \epsilon_{i}\right)\right\|_{2} \leq & \sqrt{2} \sigma \sqrt{d \log (1+2 / \epsilon)+\log \left(\frac{\operatorname{det}\left(V_{t}\right)^{1 / 2} \operatorname{det}(V)^{-1 / 2}}{\delta}\right)} \\
&\qquad\quad+\epsilon\left\|\sum_{i=1}^{t} \epsilon_{i} X_{i}^{\top} V_{t}^{-1 / 2}\right\|_{2}.
\end{aligned}
$$
By the definition of the spectral norm, this implies that,
$$
\left\|\sum_{i=1}^{t} \epsilon_{i} X_{i}^{\top} V_{t}^{-1 / 2}\right\|_{2} \leq \frac{1}{1-\epsilon} \sqrt{2} \sigma \sqrt{d \log (1+2 / \epsilon)+\log \left(\frac{\operatorname{det}\left(V_{t}\right)^{1 / 2} \operatorname{det}(V)^{-1 / 2}}{\delta}\right)}.
$$
Taking $\epsilon=1 / 2$ concludes the proof.
\end{proof}
We introduce the following lemma, which guarantees the MLE convergence refer to \cite{agarwal2020flambe}.
\begin{lemma}[MLE bound, \cite{agarwal2020flambe}]\label{lem:mle-guarantee}
By Algorithm \ref{alg:OE_low_rank}, for a fixed $t \geq 0$ and $h\in [H]$, with probability at least $1-\delta$, we have $$
\EE_{\hat{\rho}^t}[\|\cP_h^*(\cdot|s,a) - \hat{\cP}^t(\cdot|s,a)\|_1^2] \leq \frac{2\log(|\Theta||\Upsilon|/\delta)}{t}.
$$
As a straightforward corollary, we have with probability at least $1-\delta$,
$$
\EE_{\hat{\rho}^t}[\|\cP_h^*(\cdot|s,a) - \hat{\cP}^t(\cdot|s,a)\|_1^2] \leq \frac{2\log(TH|\Theta||\Upsilon|/\delta)}{t},
$$
for all $t\in [T]$ and $h\in [H]$.
\end{lemma}
The following is a standard inequality to prove regret bounds for online learning in linear models.
\begin{lemma}[\cite{agarwal2020flambe}]\label{lem:matrix-sum}
Consider the following process. For $t=1, \cdots, T, M_{t}=M_{t-1}+G_{t}$ with $M_{0}=\lambda_{0} I$ and $G_{t}$ being a positive semidefinite matrix with eigenvalues upper-bounded by $1$. We have that
$$
2 \ln \operatorname{det}\left(M_{T}\right)-2 \ln \operatorname{det}\left(\lambda_{0} I\right) \geq \sum_{n=1}^{T} \operatorname{Tr}\left(G_{t} M_{t-1}^{-1}\right).
$$
\end{lemma}
The next lemma provides an upper bound for the potential elliptical lemma and was first proved in Lemma 20 of \cite{uehara2022representation}. For completeness, we provide its proof.
\begin{lemma}[\citep{uehara2022representation}]\label{lem:log-sum}
 Suppose $\operatorname{Tr}\left(G_{n}\right) \leq B^{2}$.
$$
2 \ln \operatorname{det}\left(M_{N}\right)-2 \ln \operatorname{det}\left(\lambda_{0} I\right) \leq d \ln \left(1+\frac{N B^{2}}{d \lambda_{0}}\right) .
$$
\end{lemma}
\begin{proof}
Let $\sigma_{1}, \cdots, \sigma_{d}$ be the set of singular values of $M_{N}$ recalling $M_{N}$ is a positive semidefinite matrix. Then, by the AM-GM inequality,
$$
\ln \operatorname{det}\left(M_{N}\right) / \operatorname{det}\left(\lambda_{0} I\right)=\ln \prod_{i=1}^{d}\left(\sigma_{i} / \lambda_{0}\right) \leq d\ln \left(\frac{1}{d} \sum_{i=1}^{d}\left(\sigma_{i} / \lambda_{0}\right)\right)
$$
Since we have $\sum_{i} \sigma_{i}=\operatorname{Tr}\left(M_{N}\right) \leq d \lambda_{0}+N B^{2}$, the statement is concluded.
\end{proof}
 The next lemma provides an upper bound for the summation of potential function and is a simple generalization of the elliptical potential lemma\citep{abbasi2011improved}. In fact, it is a special case of Lemma \ref{lem:matrix-sum}.
\begin{lemma}[Elliptical Potential  Lemma]\label{lem:potential-ellip}
	For any sequence of $\{\phi_{h}(s_h^t,a_h^t)\}_{t\in[T],h\in[H]}$, we have
	$$
	\sum_{t=1}^{T} \sum_{h=1}^{H}\big\|\phi_{h}(s_h^t,a_h^t)\big\|_{(\Lambda^{t})^{-1}}^{2}\leq 2H \log \bigg(\operatorname{det}(\Lambda^{T}) \operatorname{det}(\Lambda^{0})^{-1}\bigg) .
	$$
\end{lemma}
\begin{proof}
Denote $\phi_{h}(s_h^t,a_h^t)$ by $\phi_h^t$. Recall that $\Lambda^{t+1}=\Lambda^{t}+\sum_{h=0}^{H-1} \phi_{h}^{t}\left(\phi_{h}^{t}\right)^{\top}$ and $\Lambda^{0}=\lambda I .$ Since $\lambda\geq 1$ and $\|\phi\|_2\leq 1$, $\|\phi_h^t\|_{(\Lambda^{t})^{-1}} \leq 1$ for all $(t,h)\in[T]\times[H]$. Use $x \leq$ $2H \log (1+x)$ for $x \in[0,H]$, we have
$$
\sum_{h=1}^{H}\big\|\phi_{h}^{t}\big\|_{(\Lambda^{t})^{-1}}^{2}\leq 2H \log \bigg(1+\sum_{h=1}^{H}\big\|\phi_{h}^{t}\big\|_{(\Lambda^{t})^{-1}}^{2}\bigg).
$$
For $\Lambda^{t+1}$, using its recursive formulation, we have:
$$
\log \operatorname{det}(\Lambda^{t+1})=\log \operatorname{det}(\Lambda^{t})+\log \operatorname{det}\bigg(I+(\Lambda^{t})^{-1 / 2} \sum_{h=1}^{H} \phi_{h}^{t}(\phi_{h}^{t})^{\top}(\Lambda^{t})^{-1 / 2}\bigg).
$$
Denote the eigenvalues of $(\Lambda^{t})^{-1 / 2} \sum_{h=1}^{H} \phi_{h}^{t}(\phi_{h}^{t})^{\top}(\Lambda^{t})^{-1 / 2}$ as $\sigma_{i}$ for $i \geq 1$. We have
$$
\log \operatorname{det}\bigg(I+(\Lambda^{t})^{-1 / 2} \sum_{h=1}^{H} \phi_{h}^{t}(\phi_{h}^{t})^{\top}(\Lambda^{t})^{-1 / 2}\bigg)=\log \prod_{i \geq 1}(1+\sigma_{i}) \geq \log \big(1+\sum_{i \geq 1} \sigma_{i}\big),
$$
where the last inequality uses that $\sigma_{i} \geq 0$ for all $i$. Using the above and the definition of the trace,
\begin{equation}
\begin{aligned}
\log \operatorname{det}\bigg(I+(\Lambda^{t})^{-1 / 2} \sum_{h=1}^{H} \phi_{h}^{t}(\phi_{h}^{t})^{\top}(\Lambda^{t})^{-1 / 2}\bigg) &\geq \log \bigg(1+\operatorname{tr}\bigg((\Lambda^{t})^{-1 / 2} \sum_{h=1}^{H} \phi_{h}^{t}(\phi_{h}^{t})^{\top}(\Lambda^{t})^{-1 / 2}\bigg)\bigg) \\
&=\log \bigg(1+\sum_{h=1}^{H}(\phi_{h}^{t})^{\top}(\Lambda^{t})^{-1} \phi_{h}^{t}\bigg).
\end{aligned}
\end{equation}

Telescoping over $t\in[T]$, we have
$$
\begin{aligned}
2H \sum_{t=1}^{T} \log \bigg(1+\sum_{h=1}^{H}(\phi_{h}^{t})^{\top}(\Lambda^{t})^{-1} \phi_{h}^{t}\bigg) &\leq 2H \sum_{t=1}^{T}\big(\log \operatorname{det}(\Lambda^{t+1})-\log \operatorname{det}(\Lambda^{t})\big) \\
&=2H\log \big(\operatorname{det}(\Lambda^{T}) \operatorname{det}(\Lambda^{0})^{-1}\big),
\end{aligned}
$$
Therefore, we conclude the proof of Lemma \ref{lem:potential-ellip}.
\end{proof}

The following lemma was proved in \cite{freedman_bernstein} and generalizes Bernstein's inequality for independent variables to martingale case.
\begin{lemma}[Freedman inequality]\label{lem:bernstein}
 Suppose $X_{1}, \ldots, X_{n}$ is a sequence of random variables such that $0 \leq X_{i} \leq 1$. Define the martingale difference sequence $\left\{Y_{n}=\mathbb{E}\left[X_{n} \mid X_{1}, \ldots, X_{n-1}\right]-X_{n}\right\}$ and note $K_{n}$ the sum of the conditional variances
$$
K_{n}=\sum_{k=1}^{n} \EE\left[X_{k}^2 \mid X_{1}, \ldots, X_{k-1}\right].
$$
Let $S_{n}=\sum_{i=1}^{n} X_{i}$, then for all $\epsilon, v \geq 0$,
$$
\begin{array}{r}
\mathbb{P}\left(\sum_{i=1}^{n} \mathbb{E}\left[X_{i} \mid X_{1}, \ldots, X_{i-1}\right]-S_{n} \geq \epsilon, K_{n} \leq k\right)
\leq \exp \left(-\frac{\epsilon^{2}}{2 k+2 \epsilon / 3}\right)
\end{array}
$$
\end{lemma}

\begin{lemma}[$\chi^2$-Distance Between Two Gaussians]\label{lem:chi-square-distance}
For Gaussian distributions $\mathcal{N}\left(\mu_{1}, \sigma^{2} \cI\right)$ and $\mathcal{N}\left(\mu_{2}, \sigma^{2} \cI\right)$, the (squared) chi-squared distance between $\mathcal{N}_{1}$ and $\mathcal{N}_{2}$ is,
$$
\int \frac{\left(\mathcal{N}_{1}(z)-\mathcal{N}_{2}(z)\right)^{2}}{\mathcal{N}_{1}(z)} d z=\exp \left(\frac{\left\|\mu_{1}-\mu_{2}\right\|^{2}}{2 \sigma^{2}}\right)-1.
$$

\end{lemma}
\begin{proof}
Note that,
$$
\int \frac{\left(\mathcal{N}_{1}(z)-\mathcal{N}_{2}(z)\right)^{2}}{\mathcal{N}_{1}(z)} d z=\int \mathcal{N}_{1}(z)-2 \mathcal{N}_{2}(z)+\frac{\mathcal{N}_{2}(z)^{2}}{\mathcal{N}_{1}(z)} d z=-1+\int \frac{\mathcal{N}_{2}(z)^{2}}{\mathcal{N}_{1}(z)} dz.
$$
Also note that for $\mathcal{N}_{2}^{2}(z) / \mathcal{N}_{1}(z)$, we have
$$
\mathcal{N}_{2}^{2}(z) / \mathcal{N}_{1}(z)=\frac{1}{Z} \exp \left(-\frac{1}{2 \sigma^{2}}\left(2\left\|z-\mu_{2}\right\|_{2}^{2}-\left\|z-\mu_{1}\right\|_{2}^{2}\right)\right),
$$
where $Z$ is the normalization constant for $\mathcal{N}\left(0, \sigma^{2} \cI\right)$, i.e. $Z=\int \exp \left(-\frac{1}{2 \sigma^{2}}\|z\|_{2}^{2}\right) d z$. Thus, for $2\left\|z-\mu_{2}\right\|_{2}^{2}-\left\|z-\mu_{1}\right\|_{2}^{2}$, we can verify that
$$
2\left\|z-\mu_{2}\right\|_{2}^{2}-\left\|z-\mu_{1}\right\|_{2}^{2}=\left\|z+\left(\mu_{1}-2 \mu_{2}\right)\right\|_{2}^{2}-2\left\|\mu_{1}-\mu_{2}\right\|_{2}^{2}.
$$
which implies,
$$
\begin{aligned}
\int \frac{\mathcal{N}_{2}(z)^{2}}{\mathcal{N}_{1}(z)} dz &=\frac{1}{Z} \int \exp \left(-\frac{1}{2 \sigma^{2}}\left(\left\|z-\left(2 \mu_{2}-\mu_{1}\right)\right\|_{2}^{2}-2\left\|\mu_{1}-\mu_{2}\right\|\right)\right) d z \\
&=\frac{1}{Z} \exp \left(\frac{\left\|\mu_{1}-\mu_{2}\right\|_{2}^{2}}{\sigma^{2}}\right) \int \exp \left(-\frac{1}{2 \sigma^{2}}\left\|z-\left(2 \mu_{2}-\mu_{1}\right)\right\|_{2}^{2}\right) d z \\
&=\exp \left(\frac{\left\|\mu_{1}-\mu_{2}\right\|_{2}^{2}}{\sigma^{2}}\right).
\end{aligned}
$$
Therefore, we conclude the proof.
\end{proof}
\begin{lemma}[Expectation Difference Under Two Gaussians]
For Gaussian distribution $\cN(\mu_{1}, \sigma^{2} \cI)$ and $\mathcal{N}(\mu_{2}, \sigma^{2} \cI)$, suppose that $\{\|\mu_1\|_2,\|\mu_2\|_2\}\leq B$, then for any (appropriately measurable) positive function $g$, it holds that:
$$
\mathbb{E}_{z\sim \cN_1}[g(z)]-\mathbb{E}_{z\sim \cN_2}[g(z)] \leq C(\sigma,B)\cdot\frac{\left\|\mu_{1}-\mu_{2}\right\|_2}{\sigma} \sqrt{\mathbb{E}_{\operatorname{z\sim \cN_1}}[g(z)^{2}]},
$$
where $C(\sigma,B)= \exp(B^2/\sigma^{2})$.
\end{lemma}
\begin{proof}
Define $m_{i}=\mathbb{E}_{z\sim \cN_1}[g(z)]$ for $i \in\{0,1\}$. We have:
$$\label{eq:}
\begin{aligned}
m_{1}-m_{2} &=\mathbb{E}_{z\sim \cN_1}\left[g(z)\left(1-\frac{N_{2}(z)}{N_{1}(z)}\right)\right] \\
& \leq \sqrt{\mathbb{E}_{z\sim \cN_1}\left[g(z)^{2}\right]} \sqrt{\int \frac{\left(N_{1}(z)-N_{2}(z)\right)^{2}}{N_{1}(z)} d z} \\
&=\sqrt{\mathbb{E}_{z\sim \cN_1}[g(z)^{2}]} \sqrt{\exp \bigg(\frac{\|\mu_{1}-\mu_{2}\|_2^2}{2 \sigma^{2}}\bigg)-1}.
\end{aligned}
$$
By convexity we have $\exp(x) \leq 1+x\exp(x)$ for all $x$, we have
\begin{equation*}
\begin{aligned}
\exp \bigg(\frac{\|\mu_{1}-\mu_{2}\|_2^2}{2 \sigma^{2}}\bigg)-1
&\leq \frac{\|\mu_{1}-\mu_{2}\|_2^2}{2 \sigma^{2}}\cdot \exp \bigg(\frac{\|\mu_{1}-\mu_{2}\|_2^2}{2 \sigma^{2}}\bigg)\\
&\leq \frac{\|\mu_{1}-\mu_{2}\|_2^2}{2 \sigma^{2}}\cdot \exp \bigg(\frac{2B^2}{\sigma^{2}}\bigg).
\end{aligned}
\end{equation*}
Therefore, we have
$$
m_1-m_2 \leq \exp\bigg(\frac{B^2}{\sigma^{2}}\bigg)\cdot\frac{\|\mu_{1}-\mu_{2}\|_2}{\sigma}\sqrt{\mathbb{E}_{z\sim \cN_1}[g(z)^{2}]}.
$$
\end{proof}

\vskip 0.2in
\bibliography{refs}

\end{document}